\documentclass{article}
\usepackage[margin=1.25in]{geometry}
\usepackage{microtype}
\usepackage{graphicx}
\usepackage{subfigure}
\usepackage{booktabs} %

\usepackage{fancyhdr}
\usepackage{xcolor} %
\usepackage{algorithm}
\usepackage{algorithmic}
\usepackage{natbib}
\usepackage{eso-pic} %
\usepackage{forloop}
\usepackage{url}

\usepackage{amsmath}
\usepackage{amssymb}
\usepackage{mathtools}
\usepackage{amsthm}

\usepackage{hyperref}

  \definecolor{mydarkblue}{rgb}{0,0.08,0.45}
  \hypersetup{ %
    pdftitle={},
    pdfsubject={Proceedings of the International Conference on Machine Learning 2024},
    pdfkeywords={},
    pdfborder=0 0 0,
    pdfpagemode=UseNone,
    colorlinks=true,
    linkcolor=mydarkblue,
    citecolor=mydarkblue,
    filecolor=mydarkblue,
    urlcolor=mydarkblue,
    }

\usepackage[capitalize,noabbrev]{cleveref}

\usepackage{amsthm} %
\usepackage{mathtools,thmtools}
\usepackage{amsfonts,amsmath,amssymb}
\usepackage{thm-restate}
\usepackage{tcolorbox}
\usepackage{lipsum}
\usepackage{enumitem}
\usepackage{bm}
\usepackage{xspace}
\usepackage{nicefrac}
\usepackage{dsfont}
\usepackage{float}
\usepackage{bbm}
\usepackage{mdframed}

\usepackage{multirow}

\usepackage{enumitem}
\usepackage{bm}
\usepackage{xspace}
\usepackage{nicefrac}
\usepackage{dsfont}

\DeclareMathAlphabet{\mathbfsf}{\encodingdefault}{\sfdefault}{bx}{n}

\newcommand{\lr}[1]{\mathopen{}\left(#1\right)}

\newcommand{\LR}[1]{\mathopen{}\Big(#1\Big)}

\newcommand{\LRbra}[1]{\mathopen{}\Big[#1\Big]}

\newcommand{\lrset}[1]{\mathopen{}\left\{#1\right\}}

\newcommand{\LRset}[1]{\mathopen{}\Big\{#1\Big\}}

\newcommand{\lrabs}[1]{\mathopen{}\left|#1\right|}

\newcommand{\lrceil}[1]{\mathopen{}\left\lceil #1 \right\rceil}

\newcommand{\lrfloor}[1]{\mathopen{}\left\lfloor #1 \right\rfloor}

\newcommand{\ip}[1]{\langle #1 \rangle}

\usepackage{amsfonts,amsmath,amssymb}
\usepackage{mathtools}
\usepackage{float}
\usepackage{enumitem}

\newcommand{\cH}{\mathcal{H}}

\newcommand{\cP}{\mathcal{P}}

\newcommand{\hE}{\mathbb{E}}

\newcommand{\hP}{\mathbb{P}}

\newcommand{\hide}[1]{}

\newcommand{\beq}{\begin{eqnarray*}}
\newcommand{\eeq}{\end{eqnarray*}}
\newcommand{\beqn}{\begin{eqnarray}}
\newcommand{\eeqn}{\end{eqnarray}}

\newcommand{\E}{\mathbb{E}}

\DeclareMathOperator*{\cArgmin}{arg\!\min}
\DeclareMathOperator*{\cArgmax}{arg\!\max}
\usepackage{ragged2e}

\newcommand{\Enp}{\mathbb E_{\nu, \pi}}
\newcommand{\Pnp}{\mathbb P_{\nu, \pi}}
\newcommand{\Parm}[1]{\mathbb P_{\nu_{#1}}}
\newcommand{\Earm}[1]{\mathbb E_{\nu_{#1}}}

\newcommand{\Reg}{\mathrm{Reg}}
\newcommand{\candidreg}{\mathrm{CandidReg}}

\newcommand{\alg}{\mathfrak a}
\newcommand{\env}{\nu}
\newcommand{\policy}{\pi}
\newcommand{\optarm}{a^*}

\newcommand{\numcontext}[1]{\mathbb T^{\mathcal Z}_{#1}}

\newcommand{\numaction}[1]{\mathbb T^{\mathcal A}_{#1}}

\DeclareMathOperator{\cucb}{C-UCB}
\DeclareMathOperator{\ucb}{UCB}
\DeclareMathOperator{\elim}{PE}

\newcommand{\empmeanaction}[1]{\hat\mu^{\mathcal A}_{#1}}
\newcommand{\empmeancontext}[1]{\hat\mu^{\mathcal Z}_{#1}}
\newcommand{\meanaction}[1]{\mu^{\actionspace}(#1)}
\newcommand{\meancontext}{\mu^{\contextspace}}

\newcommand{\ucbaction}[1]{\mathrm{UCB}^{\mathcal A}_{#1}}
\newcommand{\ucbcontext}[1]{\mathrm{UCB}^{\mathcal Z}_{#1}}
\newcommand{\cucbaction}[1]{\widetilde{\mathrm{UCB}}_{#1}}

\newcommand{\contextspace}{\mathcal Z}
\newcommand{\actionspace}{\mathcal A}
\newcommand{\rewardspace}{\mathcal Y}
\newcommand{\envspace}{\mathcal P(\mathcal Z\times\mathcal Y)^{\mathcal A}}
\newcommand{\benignenvspace}{\mathcal P_{\mathrm{Benign}}(\mathcal Z\times\mathcal Y)^{\mathcal A}}
\newcommand{\policyspace}{\Pi(\actionspace, \contextspace, T)}
\newcommand{\historyspace}[1]{\cH_{#1}}
\newcommand{\historyvar}[1]{H_{#1}}
\newcommand{\noncausalspace}{\mathbb A_{\mathrm{agnostic}}}

\newcommand{\indicator}{\mathbb I}
\def\[#1\]{\begin{equation}\begin{aligned}#1\end{aligned}\end{equation}}
\def\*[#1\]{\begin{equation*}\begin{aligned}#1\end{aligned}\end{equation*}}
\usepackage{mathrsfs}
\newcommand{\KL}[2]{\mathrm{KL}(#1\ \Vert \ #2)}
\newcommand{\bernoulli}[1]{\mathrm{Ber}(#1)}

\newcommand{\ucbactionrv}[1]{A^{\mathrm{UCB}}_{#1}}
\newcommand{\cucbactionrv}[1]{A^{\mathrm{C-UCB}}_{#1}}

\newcommand{\eventaction}{E^{\mathcal A}}
\newcommand{\eventcontext}{E^{\mathcal Z}}
\newcommand{\eventmartingale}{E^{\mathrm{MG}}}
\newcommand{\eventdb}{\mathcal E}

\newcommand{\phaselength}[1]{T_{#1}}

\newcommand{\activeset}[1]{\mathcal A_{#1}}
\newcommand{\dimspan}[1]{d_{#1}}
\newcommand{\lastphase}{\ell_{\text{max}}}

\newcommand{\phaseestimate}[1]{\hat\mu^{\contextspace}_{#1}}
\newcommand{\phasecov}[1]{V_{#1}}
\newcommand{\phasedesign}[1]{\pi_{#1}}
\newcommand{\supp}[1]{\mathrm{supp}(#1)}
\newcommand{\trueoptarm}[1]{\optarm_{#1}}
\newcommand{\empoptarm}[1]{\hat a_{#1}}
\newcommand{\phaseconcentration}[1]{E^{\mathrm{phase}}_{#1}}
\newcommand{\phase}[1]{\text{phase }#1}

\newcommand{\actionnoise}[1]{\eta^{\actionspace}_{#1}}

\newcommand{\mingap}{\Delta_{\min}}

\newcommand{\localreward}[2]{U_{#1}(#2)}
\newcommand{\localtime}[2]{n_{#1}(#2)}
\newcommand{\activelearner}[1]{\mathcal I_{#1}}
\newcommand{\regbalancing}{\mathrm{DB}}

\DeclareMathOperator{\hac}{HAC-UCB}

\newcommand{\defn}[1]{\emph{#1}}
\newcommand{\PosReals}{\Reals_{>0}}

\newcommand{\Reals}{\mathbb R}

\usepackage{crossreftools}

    \makeatletter
\def\@fnsymbol#1{\ensuremath{\ifcase#1\or \dagger\or \ddagger\or
   \mathsection\or \mathparagraph\or \|\or **\or \dagger\dagger
   \or \ddagger\ddagger \else\@ctrerr\fi}}
    \makeatother

\theoremstyle{plain}
\newtheorem{theorem}{Theorem}[section]
\newtheorem{proposition}[theorem]{Proposition}
\newtheorem{lemma}[theorem]{Lemma}
\newtheorem{corollary}[theorem]{Corollary}
\theoremstyle{definition}
\newtheorem{definition}[theorem]{Definition}

\theoremstyle{remark}
\newtheorem{remark}[theorem]{Remark}

\usepackage[disable]{todonotes}
\usepackage[setmargin=true,marginparwidth=0.8in,hide=true]{marginalia}

\title{
Causal Bandits: The Pareto Optimal Frontier of Adaptivity, a Reduction to Linear Bandits, and Limitations around Unknown Marginals}

\date{}

\author{
    Ziyi Liu $^\star$ \thanks{Department of Statistical Sciences,
University of Toronto and Vector Institute; 
\texttt{kevind.liu@mail.utoronto.ca}.}
    \and Idan Attias $^\star$ \thanks{Department of Computer Science, Ben-Gurion University and Vector Institute; \texttt{idanatti@post.bgu.ac.il}.} 
    \and Daniel M. Roy\thanks{Department of Statistical Sciences,
University of Toronto and Vector Institute; \texttt{daniel.roy@utoronto.ca}.}
}

\begin{document}
\thispagestyle{empty}

\maketitle

\def\thefootnote{$\star$}\footnotetext{Equal contribution.}
\def\thefootnote{\arabic{footnote}}

\begin{abstract}
In this work, we investigate the problem of adapting to the presence or absence of causal structure in multi-armed bandit problems. In addition to the usual reward signal, we assume the learner has access to additional variables, observed in each round after acting. When these variables $d$-separate the action from the reward, existing work in causal bandits demonstrates that one can achieve strictly better (minimax) rates of regret (Lu et al., 2020).  
Our goal is to adapt to this favorable ``conditionally benign'' structure, if it is present in the environment, while simultaneously recovering worst-case minimax regret, if it is not. Notably, the learner has no prior knowledge of whether the favorable structure holds. In this paper, we establish the Pareto optimal frontier of adaptive rates. We prove upper and matching lower bounds on the possible trade-offs in the performance of learning in conditionally benign and arbitrary environments, resolving an open question raised by Bilodeau et al. (2022). Furthermore, we are the first to obtain instance-dependent bounds for causal bandits, by reducing the problem to the linear bandit setting. Finally, we examine the common assumption that the marginal distributions of the post-action contexts are known
and show that a nontrivial estimate is necessary for better-than-worst-case minimax rates.
\end{abstract}

\section{Introduction}\label{sec:intro}

In real-world decision making, we often want strong worst-case guarantees as well as the ability to adapt to favorable properties of real-world scenarios.
Adaptive sequential decision-making offers a framework to design algorithms to achieve these objectives. 

In this paper, we explore adaptivity in multi-armed bandit problems. In standard multi-armed bandits, the learner (policy) takes an action, receives a reward, and then this process repeats over a number of rounds. The learner's regret is the difference between its cumulative reward and the cumulative reward of the single best action in hindsight. Can we work to identify high-reward actions while minimizing regret?

In this work, we assume there is \emph{post-action context}, i.e., there may be additional information available to the learner after taking an action, beyond the reward signal. In a worst-case analysis, however, the learner can ignore the post-action context and still achieve minimax rates of regret: the worst-case environment will not offer useful information. However, many real-world settings possess the structure of multi-armed bandit problems with post-action context and, in those cases, this additional information is useful towards minimizing regret.

One way that post-action context can be useful is if we can assume causal structure relating the action (i.e., an intervention) to the reward and post-action (post-intervention) context.
Several authors have studied models in this vein \citep{BareinboimEtAl2015,lattimore2016causal}.
In this work, we build on the framework of \citet{lattimore2016causal}, wherein the post-action context is assumed to $d$-separate each intervention from its associated reward.

Under $d$-separation, 
the intervention and reward are independent, 
conditional on the post-intervention context. 
\citet{bilodeau2022adaptively} formalized this structure
in general terms: a bandit environment is \emph{conditionally benign} whenever the conditional distribution of the reward, given the post-action context, does not depend on the action.

Minimax regret is well understood for both the classical and causal variant of multi-armed bandits. Notably, algorithms tailored to conditionally benign environments can achieve lower rates of regret, scaling with the number of post-action contexts, rather than the potentially much larger set of actions \citep{lu2020regret,bilodeau2022adaptively}. 

Exploiting causal structure is not without its pitfalls. 
\citet{bilodeau2022adaptively} showed that C-UCB, a minimax optimal causal bandit algorithm, suffers linear regret in some non-benign environments.
This raised a natural question: Can we achieve strict adaptivity, 
i.e., obtain minimax rates simultaneously in the class of conditionally benign environments and in the class of all environments, without knowing in advance which class of environments we will face?

\citeauthor{bilodeau2022adaptively} proved that \emph{strict} adaptivity was impossible, but showed some level of adaptivity was possible. They designed a new algorithm, termed HAC-UCB, and proved that it simultaneously achieves minimax optimal rates on the class of benign environments and always achieves (suboptimal, though sublinear) $T^{3/4}$ rates. In light of this result, \citeauthor{bilodeau2022adaptively} raised an open problem, asking whether HAC-UCB was, in a sense, Pareto optimal, implying  that the slower rate was the price of adaptivity. More generally, we ask: 
\begin{center}
    \emph{What is the Pareto optimal frontier of simultaneously achievable rates of regret in the classes of benign and arbitrary environments, and what algorithms achieve these optimal tradeoffs?}
\end{center}

In this paper, we address the above question by providing a complete characterization of the Pareto optimal frontier (up to log factors) as well as the achieving algorithms. Besides adaptation, we also study the complexity of causal bandit problems from other perspectives. More specifically, we find a novel reduction from causal bandits to linear bandits, which facilitates the first instance-dependent regret bound for causal bandits and enables the applications of some linear bandit algorithms to causal bandits. We also investigate dropping the common assumption that we have perfect knowledge of ``the marginals", i.e., the distribution of the post-action context variable, under each action. On one hand, we show that it is impossible for any algorithm to enjoy improved minimax regret in benign environments without any knowledge of the true marginals. On the other hand, we identify cases where approximate knowledge of the marginal distributions suffices. Our contributions are explained in more details as follows. 
\begin{itemize}[leftmargin=0.3cm]
    \item In \cref{sec:pareto}, we establish near-optimal Pareto regret frontiers for the setting of causal bandits, resolving an open problem raised by \citet{bilodeau2022adaptively}, see \Cref{fig:pareto}.
    Utilizing a dynamic balancing method introduced by \citet{cutkosky21a}, we derive the upper bound and also prove near-optimal matching lower bounds. Remarkably, we introduce a phenomenon we call \emph{the price of adaptivity}, to capture the extra regret that one \emph{must} incur when attempting to adapt to the presence or lack of causal structure. Consequently, we demonstrate that the model selection method introduced by \citet{cutkosky21a} cannot be generally improved, for any nontrivial general improvement would decrease the price of adaptivity beyond our lower bound.
    \item In \cref{sec:benign}, we present a novel reduction from causal bandits to linear bandits with conditional sub-Gaussian noise. Utilizing a phased elimination technique \citep{lattimore20a},
    we identify a new dimension measuring the inherent complexity of causal bandits. It allows us to establish the first instance-dependent regret bound and a strictly tighter worst-case regret bound for causal bandits for conditionally benign environments. Additionally, we prove instance-dependent bounds for stochastic linear bandits, which are novel to the best of our knowledge.
    \item In \cref{sec:approx-marginal}, we study the situation where we have limited knowledge of the marginal distributions over post-action contexts. We provide a lower bound indicating that no algorithm can utilize the causal structure to achieve improved minimax rates without such prior knowledge. This partly justifies the common assumption in the causal bandits literature that algorithms are given the marginals. On the other side, we give a regret upper bound for the phased elimination algorithm with access to approximate marginals. This result shows that partial knowledge of the marginals suffices in some regimes.
\end{itemize}
\begin{figure}[]
    \centering
    \includegraphics[width=0.75\textwidth]{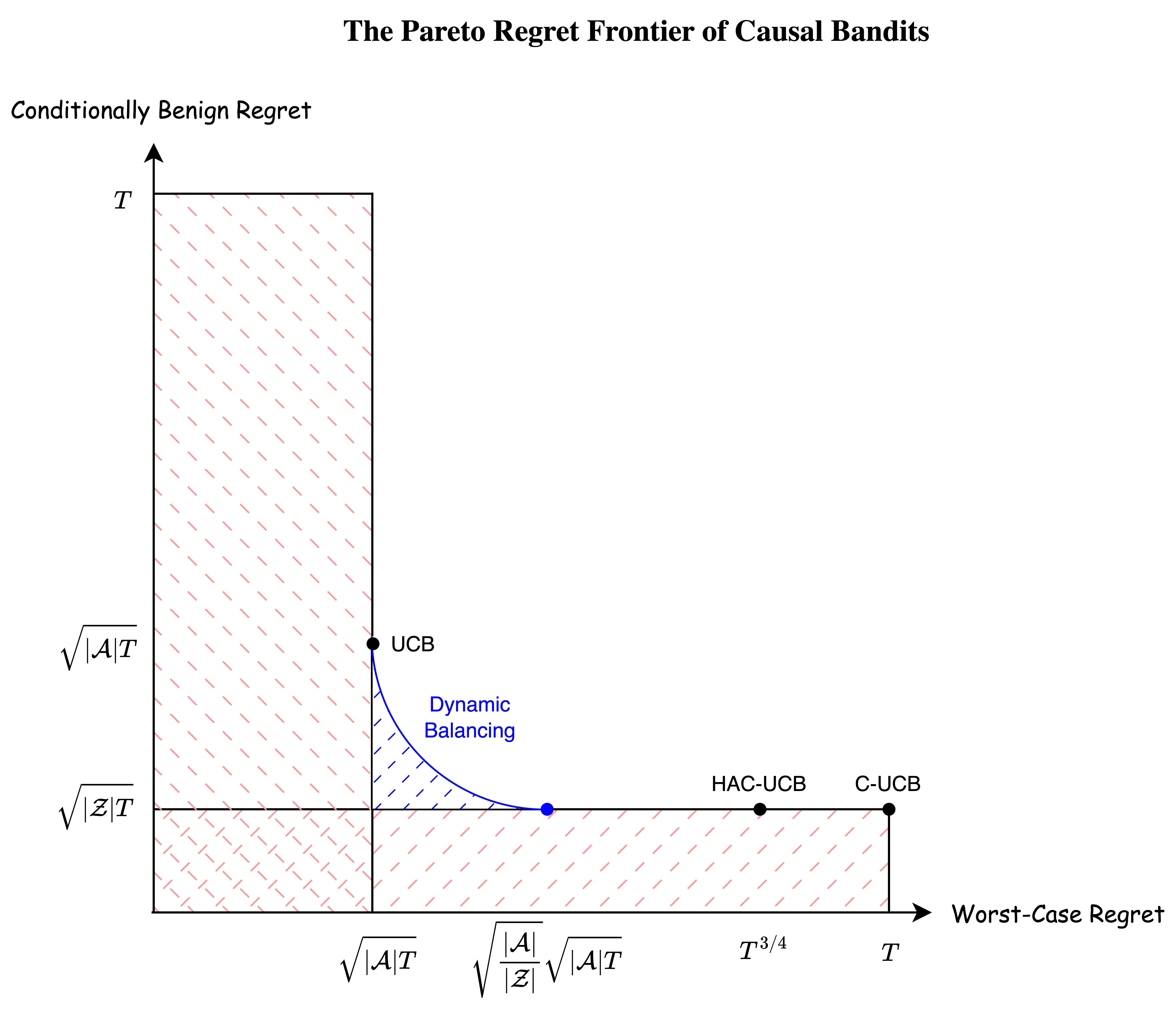}
    \caption{The Pareto-optimal frontier of simultaneously achievable rates of regret in (left axis) the class of conditionally benign environments and (bottom axis) the class of all environments. Shaded regions are unobtainable. All rates are determined up to log terms. Among algorithms that achieve minimax rates on conditionally benign environments, the previously best known algorithm (HAC-UCB) is dominated by an instance of Dynamic Balancing, which our results also demonstrate is Pareto optimal.}
    \label{fig:pareto}
\end{figure}

\subsection{Related Work} 
\paragraph{Causal bandits.} 
The causal bandit model was introduced by \citet{lattimore2016causal}, where their objective was to identify the best intervention. Such pure exploration problem has been extensively studied since then \citep{sen2017identifying,xiong2022pure}, while some other works focused on regret minimization \citep{lu2020regret, nair2021budgeted, bilodeau2022adaptively}. Another interesting topic is to relax the causal assumptions. For example, the assumption of known causal graph can be relaxed \citep{lu2021causal, malek2023additive}. Our work mainly builds on the study by \citet{bilodeau2022adaptively} regarding adapting to the existence of causal structures as well as approximate marginals.

\paragraph{Model selection.} To achieve adaptivity, a natural idea is to apply some model selection algorithm on top of a group of base learners. There is an extending line of works studying such corralling strategies in the bandit setting \citep{agarwal2017corralling, pacchiano2020regret, pacchiano2020model, cutkosky2020upper, arora2021corralling, cutkosky21a}. \citet{agarwal2017corralling} required certain stability conditions on the base learners, making their algorithm quite restricted. In contrast, some recently proposed general-purpose model selection algorithms for stochastic bandit problems \citep{pacchiano2020model, cutkosky2020upper, cutkosky21a} are better candidates in our setting, since they only necessitate mild assumptions on the base learners.
 
\paragraph{Pareto optimal frontier.} When we have multiple performance metrics but are unable to achieve the best under all of them simultaneously, the Pareto optimal frontier becomes a common objective to pursue subsequently. Problems with several competing benchmarks are abundant in bandit literature \citep{koolen2013pareto, lattimore2015pareto, marinov2021pareto, zhu2022pareto}.

\section{Problem Setup}\label{sec:prob-setup}
We consider the problem of stochastic bandit with post-action contexts, as defined by \citet{bilodeau2022adaptively} and follow their notations. Let $\actionspace$ be the finite action space, $\contextspace$ be the finite context space and $\rewardspace=[0,1]$ be the reward space. For any set $K$, we use $\cP(K)$ to denote the set of all probability distributions supported on $K$. For any $p\in\cP(\contextspace\times\rewardspace)$, we use $p(Z)$ to denote its marginal distribution over $\contextspace$, and use $p(Y|Z)$ to denote its the conditional distribution over $\rewardspace$ conditioning on the $Z-$component.

In this bandit problem, a learner interacts with the stochastic environment for $T$ rounds. The role of the environment is instantiated with a family of distributions $\env=\{\env_a: a\in\actionspace \}\in\cP(\contextspace\times\rewardspace)^{\actionspace}$ indexed by actions in $\actionspace$. For each round $t\in[T]$, the learner picks an action $A_t$ from $\actionspace$ and then receives a context-reward pair $(Z_t, Y_t)$ which is independently sampled from $\env_{A_t}\in\cP(\contextspace\times\rewardspace)$.

To model learner's strategy, we need to formalize the information that can be used for learner's prediction. Let $\historyvar{t}=(A_s, Z_s, Y_s)_{s\in[t]}$ denote the observed history up to round $t$, which is a random variable valued in $\historyspace{t}:=(\actionspace\times\contextspace\times\rewardspace)^t$. A policy $\pi$ by the learner could be modeled as a sequence of measurable maps from $\historyspace{t}$'s to $\actionspace$
 \*[
 \policy
 =
 (\policy_t)_{t\in[T]}\in\policyspace
 := 
 \prod_{t=1}^T \{\historyspace{t-1}\to\actionspace \},
 \]
where $\policyspace$ is the space of all policies compatible with $(\actionspace,\contextspace, T)$. Then the learner follows this policy by selecting $A_t=\policy_t(\historyvar{t-1})$ for each round $t$. Indeed, the distribution of all outcomes over $T$ rounds, i.e. $(A_t, Z_t, Y_t)_{t\in[T]}$, is determined by the environment $\env$ and the player's policy $\policy$ together. We will always highlight the ambient joint distribution by the subscript on probabilistic operators $\hP$ and $\E$, say $\Earm{a}$ and $\Enp$. Additionally, we denote the expected reward for action $a$ and the optimal action $\optarm$ by

 \[
 \meanaction{a}
 :=
 \Earm{a}[Y], 
 \quad
 \optarm
 :=
 \cArgmax_{a\in\actionspace}\meanaction{a}.
 \]

The goal of the learner is to choose some policy $\policy$ that maximizes her expected cumulative reward $\Enp[\sum_{t=1}^T Y_t]$, or equivalently minimizes her expected pseudo-regret
 \begin{align*}
    \Enp[\Reg(T)]
    :=
    \Enp[\sum_{t=1}^T \max_{a\in\actionspace}\Earm{a}[Y]-Y_t]
    =
    T\cdot\meanaction{\optarm}-\Enp[\sum_{t=1}^T \meanaction{A_t}],  
 \end{align*}
with 
$
\Reg(t)
:=
t\cdot\meanaction{\optarm} - \sum_{s=1}^t \meanaction{A_s}, t\in[T]
$
being the realized regret, \NA{which is stochastic.}

\paragraph{Conditionally benign property and $d$-separation.}
Under certain structures, the post-action context variable $Z$ enables more efficient exploration and hence smaller regret. One special structure that can be exploited for better regret guarantee in our setting is called \textit{conditionally benign property}, introduced by \citet{bilodeau2022adaptively}. 
\begin{definition}\citep[][Definition~3.1]{bilodeau2022adaptively}\label{def: conditionally benign property}
    An environment $\env\in\envspace$ is \textit{conditionally benign} if and only if there exists $p\in\cP(\contextspace\times\rewardspace)$ such that for each $a\in\actionspace$, $\env_a(Z)\ll p(Z)$ and $\env_a(Y|Z)=p(Y|Z)$ p-a.s. We further denote the space of all conditionally benign environments by $\benignenvspace$.
\end{definition}
The conditional benign property is quite general in the sense that it is equivalent to or weaker than some well-studied causal assumptions (\citealt{bilodeau2022adaptively}). In particular, the conditionally benign property is the same thing as the context variable $Z$ being a $d$-separator when $\actionspace$ is all interventions. To leverage this benign structure, the causal UCB ($\cucb$) algorithm recently proposed by \citet{lu2020regret} achieves $\tilde O(\sqrt{|\contextspace|T})$ regret, while non-causal algorithms that is unaware of this structure would still incur the possibly worse regret of $\tilde O(\sqrt{|\actionspace|T})$.
\subsection{Adaptivity}
A natural question is whether we can compete with C-UCB when the environment is conditionally benign while at the same time still maintain the worst-case $\tilde O({\sqrt{|\actionspace|T}})$ regret guarantee, without prior knowledge of the nature of the environment. Unfortunately algorithms designed specific to the benign setting may fail drastically in non-benign settings. For instance, C-UCB provably incurs linear regret in some non-benign environments \citep{bilodeau2022adaptively}. To remedy this, \citet{bilodeau2022adaptively} devised $\hac$ by adding a hypothesis test in each round, which is used for switching away from C-UCB to UCB irreversibly whenever it detects a deviation from conditionally benign property. $\hac$ is able to recover the $\tilde O(\sqrt{|\contextspace|T})$ regret in benign settings and achieve sublinear $\tilde O(T^{3/4})$ regret in the worst case.

Prior to this work, 
we do not know if $\hac$ is optimal. Indeed, \citet{bilodeau2022adaptively} showed that \textit{strict adaptation}, meaning that always achieving the worst-case $O(\sqrt{|\actionspace|T})$ regret while still being able to perform as good as $\cucb$ when causal structure exists, is impossible. But this does not rule out the possibility of improving the worst-case $\tilde O(T^{3/4})$ regret of $\hac$ unilaterally. In this paper we will show that such improvement is indeed feasible and thus obtain an algorithm that dominates $\hac$. Further we will show that our regret guarantee is not improvable through the lens of Pareto optimality.

\begin{remark}
    Regarding optimal rate of regret under the presence of causal structure, it is easy to show a  $\Omega(\sqrt{|\contextspace|T})$ regret lower bound, nearly matching existing $\tilde O(\sqrt{|\contextspace|T})$ regret upper bounds. Whether the log-factors can be shaved from the upper bound is unknown. However, the lower bound of \citet{bilodeau2022adaptively} still implies that strict adaptation is impossible for general $\actionspace$ and $\contextspace$, since when $|\actionspace|/|\contextspace|$ is, say, $\Omega(T^{1/5})$, the $\Omega(\sqrt{|\actionspace|T})$ lower bound in benign settings rules out a $\tilde O(\sqrt{|\contextspace|T})$ upper bound.
\end{remark}
\paragraph{Generic algorithms.} For rigorous treatment of adaptivity, we adopt the definition of algorithms as maps from \citet{bilodeau2022adaptively}. Specifically, an algorithm $\alg$ is any map from problem-specific inputs to the space of compatible policies
\*[
\alg: (\actionspace, \contextspace, T, q) 
\mapsto
\alg(\actionspace, \contextspace, T, q)\in\policyspace,
\]
where $q\in\cP(\contextspace)^{\actionspace}$ is the marginal distribution accessed by this algorithm as prior knowledge. When talking about algorithm-induced policies, by default we mean $\alg(\actionspace, \contextspace, T, \env(Z))$ if not stated otherwise, following the common assumption in the literature of causal bandits. We will also deal with the case of imperfect prior knowledge in \cref{sec:approx-marginal}, where $q$ may not be the exact $\env(Z)$. For notation simplicity, we will use $\alg$ to denote its induced policy $\alg(\actionspace, \contextspace, T, q)$ when the problem-specific inputs are clear from context. For example, $\hE_{\env, \alg}$ is the same thing as $\hE_{\env, \alg(\actionspace, \contextspace, T, q)}$.

\section{The Pareto Regret Frontier}\label{sec:pareto}
To formalize our notion of Pareto regret frontier, we need the following definition: 
\begin{definition}\label{def: realizable pair}
    A pair of rate functions $(R_1(T;\actionspace,\contextspace), R_2(T;\actionspace,\contextspace))$ is said to be \emph{realizable} if there is an algorithm $\alg$ such that for all $\actionspace,\contextspace$ and $T$,
    \*[
    \sup_{\env\in\benignenvspace}\E_{\env,\alg}[\Reg(T)]
    \leq
    R_1(T;\actionspace,\contextspace), \\
    \sup_{\env\in\envspace}\E_{\env,\alg}[\Reg(T)]
    \leq
    R_2(T;\actionspace,\contextspace).
    \]
    A pair $(R_1(T;\actionspace,\contextspace), R_2(T;\actionspace,\contextspace))$ is  \emph{reasonable} if $R_1(T;\actionspace,\contextspace)\geq\sqrt{|\contextspace|T}$ and $R_2(T;\actionspace,\contextspace)\geq\sqrt{|\actionspace|T}$.
\end{definition}
In the following we elide the dependence of rates $R_i$ on $\actionspace$ and $\contextspace$ below for clarity. We can now describe the Pareto regret frontier, i.e., the set of optimal realizable pairs of rates.
\begin{theorem}\label{main theorem: pareto regret frontier}
    There exists universal constants 
    $C, c, c'>0$ such that
    \begin{enumerate}[leftmargin=1.5em]
    \item Upper bound:
    If $(R_1(T), R_2(T))$ is reasonable and $R_1(T)R_2(T)\geq|\actionspace|T$, then
    $(C R_1(T)\log T, C R_2(T) \log T)$ is realizable;
    \item Lower bound:
    For all realizable $(R_1(T), R_2(T))$, we have  $R_2(T)>c'T$ or $R_1(T)R_2(T)\geq c|\actionspace|T$.
    \end{enumerate}
\end{theorem}
\fLATER{DR: Make nicer format for the above theorem.}

Both upper and lower bounds will be extensively discussed in the following sections.

\subsection{Upper Bounds}\label{subsec: pareto upper bound}
In this section, we show that our upper bound can be obtained by applying the algorithmic principle of \emph{dynamic balancing} ($\regbalancing$) in \citet{cutkosky21a} to the stochastic bandit problem with post-action contexts. This method is motivated by the fact that, under mild assumptions, it can always achieve $\tilde O(\sqrt{T})$ regret when it is running on top of a collection of $\tilde O(\sqrt{T})$ regret base learners. So the dependence on $T$ in $\tilde O(T^{3/4})$ regret by $\hac$ in 
\citet{bilodeau2022adaptively} is easily improved.
The use of dynamic balancing in our bandit setting can be justified by the fact that 
dynamic balancing does not rely on what kind of (stochastic) contextual information can be observed in the underlying bandit problem. 
See \cref{appendix sec: regret analysis for DB} for a detailed explanation.

\begin{algorithm}[t]
    \caption{Dynamic balancing ($\regbalancing$) w/ two base learners}\label{alg: regret balancing}
    \textbf{Input:} Two base learners, $\{\alg_i\}_{i = 1,2}$, factor $d_i(\cdot)$ of candidate regret bound, reward bias $b_i(\cdot)$ and scaling coefficient $v_i$ (hyper-parameters) for each base learner $i\in\{1,2\}$, and confidence level $\delta\in(0,1)$.
    \begin{enumerate}
        \item Set $\localreward{i}{0}=\localtime{i}{0}=0$ for all $i\in\{1,2\}$ and let the set of active learners be $\activelearner{1}=\{1,2\}$
        \item \textbf{For} $t=1,2,...,T$ \textbf{do}
        \begin{enumerate}[itemsep=0.5em,leftmargin=1em]
            \item Select learner from the active set: 
            $
                i_t\in\cArgmin_{i\in\activelearner{t}} v_i d_i(\delta)\sqrt{\localtime{i}{t-1}}
            $
            \item Play action $A_t$ of learner $\alg_{i_t}$ and receive reward $Y_t$ and context $Z_t$
            \item Update learner $\alg_{i_t}$ with $Z_t$ and $Y_t$
            \item Update $\localtime{i}{\cdot}$ and $\localreward{i}{\cdot}$: \\
            $
            \localreward{i}{t}
            \leftarrow\localreward{i}{t-1}+Y_t\indicator\lrset{i=i_t}
            $
            \\
            $
            \localtime{i}{t}
            \leftarrow\localtime{i}{t-1}+\indicator\lrset{i=i_t}
            $
            \item Compute adjusted average reward $\eta_i(t)$ and confidence band $\gamma_i(t)$ for all $i\in\{1,2\}$:\\
            $
            \eta_i(t)
            \leftarrow\frac{\localreward{i}{t}}{\localtime{i}{t}}-b_i(t)
            $
            \\
            $
            \gamma_i(t)
            \leftarrow 3\sqrt{\frac{\log(2\log\localtime{i}{t}/\delta)}{\localtime{i}{t}}}
            $
            \item Update the set of active learners:\\
            $
            \activelearner{t+1}\leftarrow\Bigl\{i\in\{1,2\}: \eta_i(t)+\gamma_i(t)+\frac{d_i(\delta)}{\sqrt{\localtime{i}{t}}}
            \geq 
            \max_{j=1,2} \eta_j(t)+\gamma_j(t) \Bigl\}
            $
        \end{enumerate}
    \end{enumerate}
\end{algorithm}
Note that dynamic balancing algorithm (\cref{alg: regret balancing}) is input by a set of user-specified candidate regret bounds for each base learner $i$ (which takes the form of $d_i\sqrt{t}$ in our setting). In each round, $\regbalancing$ merely picks the base learner with minimal candidate regret bound, and performs a test to identify and deactivate the learners that seem to violate their candidate regret bounds. As long as there is one base learner whose candidate regret is valid, $\regbalancing$ is able to compete with the best of such base learners. A more comprehensive exposition of the idea behind dynamic balancing can be found in \citet{cutkosky21a}.

So naturally, we need one base learner that is favorable in benign instances and another base learner that remains robust to non-benign instances. For example, we can pick $\cucb$ and $\ucb$, but note that any other algorithm with similar regret bound can be applied as well. Formally we characterize base learners that enjoys certain regret bound in certain type of environments by the following definition:
\begin{definition}\label{def: benign and arbitrary family}
Let $d: (0,1) \to \PosReals$.
A family of learners $\alg=(\alg_\delta)_{\delta \in (0,1)}$ is a \defn{$d$-benign family} if, for all $\delta \in (0,1)$, for all benign instances,
with probability at least $1-O(\delta)$,
for all $t \in [T]$, $\alg_\delta$ has regret no larger than $d(\delta) \sqrt{t}$.
Similarly, a learner $\alg$ is a \defn{$d$-arbitrary learner} if, for all $\delta \in (0,1)$, for all instances, with probability at least $1-O(\delta)$,
for all $t \in [T]$, $\alg$ has regret no larger than $d(\delta) \sqrt{t}$.
\end{definition}
Let $\cucb=(\cucb(\delta))_{\delta\in(0,1)}$  and $\ucb=(\ucb(\delta))_{\delta\in(0,1)}$ be the families of instances of the $\cucb$ and $\ucb$ algorithms, respectively, whose confidence band is scaled by $\Theta(\sqrt{\log(1/\delta)})$. 
See \cref{appendix sec: anytime regret for ucb and cucb} for  details.
\begin{proposition}\label{regret bound: cucb and ucb}
     $\cucb$ is a $d$-benign family for $d(\delta)=O((\sqrt{|\contextspace|}+\sqrt{\log (T/\delta)})\sqrt{\log(|\contextspace|T/\delta)})$
     and $\ucb$ is a $d'$-arbitrary family for 
     $d'(\delta)= O(\sqrt{|\actionspace|\log(|\actionspace|T/\delta)})$.
\end{proposition}
Note that the above result for $\ucb$ is folklore, but the result for $\cucb$ is new.
The following result describes the adaptive regret of dynamic balancing acting on a benign family and an arbitrary family, which validates the upper bound in \cref{main theorem: pareto regret frontier}. What is more impressive is that to realize every point on the Pareto regret frontier (up to log factors), we need only tune the hyper-parameters in $\regbalancing$ accordingly.
We elide the dependence of rates $R_i$ on $\actionspace$ and $\contextspace$ below for clarity. 
See \cref{appendix sec: proofs for section 3.1} for the proof.

\begin{theorem}\label{regret bound: regret balancing general rate}
    Let $\alg_1$ be a $d_1$-benign family and let $\alg_2$ be a $d_2$-arbitrary family of learners, where $d_1(\delta)=O\lr{(\sqrt{|\contextspace|}+\sqrt{\log (T/\delta)})\sqrt{\log(|\contextspace|T/\delta)}}$,
    $d_2(\delta)= O(\sqrt{|\actionspace|\log(|\actionspace|T/\delta)})$.
    For every pair of reasonable rate functions $R_1(T), R_2(T)$ such that $R_1(T)R_2(T)\geq |\actionspace|T$,    
    there exist hyper-parameters $b_i(\cdot), v_i$, $i=1,2$,
    such that, 
        for all instances $\env$,  
    the policy $\regbalancing(\delta)$, for $\delta=1/T$,
    given by 
    \cref{alg: regret balancing} with $\alg_1,\alg_2$ and $d_1, d_2$,
    satisfies 
    \*[
    \E_{\env,\regbalancing(\delta)}[\Reg(T)]
    &=
    \tilde O(R_1(T))
    ;
    \;
    \env \text{ is conditionally benign,}\\
    \E_{\env,\regbalancing(\delta)}[\Reg(T)]
    &=
    \tilde O(R_2(T))
    ;
    \;
    \env \text{ is arbitrary.}\\
    \]
\end{theorem}

\begin{corollary}\label{regret bound: regret balancing common rate}
     Taking $R_1(T)=\sqrt{|\contextspace|T}$ and $R_2(T)=\sqrt{|\actionspace|/|\contextspace|}\cdot\sqrt{|\actionspace|T}$, the conclusion of \cref{regret bound: regret balancing general rate} is
    \*[
    \E_{\env,\regbalancing(\delta)}[\Reg(T)]
    &=
    \tilde O(\sqrt{|\contextspace|T})
    ;
    \;
    \env \text{ is conditionally benign,}\\
     \E_{\env,\regbalancing(\delta)}[\Reg(T)]
     &=
     \tilde O(\sqrt{|\actionspace|/|\contextspace|}\cdot\sqrt{|\actionspace|T})
    ;
    \;
     \env \text{ is arbitrary.}
    \]
\end{corollary}
\cref{regret bound: regret balancing common rate} indicates that we need to pay an extra factor of $\sqrt{|\actionspace|/|\contextspace|}$ in the worst-case regret for adaptivity, and it already improves over the one by $\hac$ in terms of worst-case regret. Moreover, our regret analysis does not require their cumbersome assumption that $T\geq 25|\actionspace|^2$. Such improvement may be explained as follows. Both dynamic balancing and $\hac$ play with two base learners and decide which to pick in each round. However, $\regbalancing$ is operating in a more reasonable way: $\regbalancing$ alternates between two base learners and never deactivates any of them permanently, whereas $\hac$ first plays the optimistic base learner persistently up to some point and then switches to $\ucb$ for the remaining rounds. Thus the regret of $\hac$ incurred by running the optimistic base learner improperly may be dominant.

\subsection{Lower Bounds}\label{subsec: pareto lower bound}

In this section we elaborate on the lower bound in \cref{main theorem: pareto regret frontier} in the following \cref{lower bound: pareto lower bound}, which is a generalization of \citep[][Theorem~6.2]{bilodeau2022adaptively}. The proof of \cref{lower bound: pareto lower bound} closely follows that of the original, but we are able to derive a continuum of lower bounds that constitute the Pareto regret frontier. For completeness, we provide the full proof in \cref{appendix sec: proof of lower bound on the frontier}.

\begin{theorem}\label{lower bound: pareto lower bound}
    There exists constants $c, c'>0$ such that, for all MAB algorithms $\alg$, rate functions $R(T;\actionspace,\contextspace)$,
    if,
    for all $\actionspace,\contextspace, T$
    \*[
    \sup_{\env} \E_{\env,\alg}[\Reg(T)]
    \leq
    R(T;\actionspace,\contextspace),
    \]
     then, for all $\actionspace, \contextspace$ and $T$, there exists a conditionally benign environment $\env$ such that either
     $
     R(T;\actionspace,\contextspace)>c'T,
     $
     or there exists a conditionally benign environment $\env$ such that
    \*[
    \E_{\env,\alg}[\Reg(T)]
    \geq
    c\cdot \frac{|\actionspace|T}{R(T;\actionspace,\contextspace)}.
    \]
\end{theorem}
\cref{lower bound: pareto lower bound} shows that any pair of realizable rates must have their product lower bounded by $|\actionspace|T$ unless the worst-case regret bound is vacuously large.
Combining \cref{regret bound: regret balancing general rate} with \cref{lower bound: pareto lower bound}, we have justified the Pareto optimality of dynamic balancing. As a corollary, we have found a problem of adaptation where model selection method can be optimal and the price of adaptivity is witnessed by the additional multiplicative factor of $\sqrt{|\actionspace|/|\contextspace|}$ in the regret bound.

\section{Instance-Dependent Bounds via Phased Elimination Algorithm }\label{sec:benign}

Besides achieving Pareto optimal regret bounds in \cref{regret bound: regret balancing general rate} that are worst-case in nature, the dynamic balancing algorithm can also enjoy $O(\log T)$ instance-dependent regret at the same time under additional assumptions on the base learners. In particular, $\cucb$ may not be our best choice for the benign base learner. To leverage the strength of dynamic balancing, we propose a new causal bandit algorithm that enjoys $\tilde O(\sqrt{|\contextspace|T})$ worst-case regret and a novel logarithmic instance-dependent regret in benign settings in this section. We are the first to pursue  instance-dependent results in conditionally benign environments
for algorithms
that are minimax optimal (up to log factors). 

Our new algorithm is built upon the idea of \emph{phased elimination with G-optimal design} from linear bandits \citep{lattimore2020bandit,lattimore20a}. Our regret analysis hinges on a novel reduction from causal bandits to linear bandits. This reduction enables the use of a broad family of linear bandit algorithms in conditionally benign environments, whose regret guarantees remain intact.

Finally, we will discuss the possibilities and challenges regarding adaptive $O(\log T)$ instance-dependent regret.

\subsection{Reduction to Linear Bandits}\label{subsec: causal to linear reduction}
We need additional notations to illustrate our causal-to-linear reduction. For benign instance $\env$, define the mean reward vector $\meancontext\in[0,1]^{|\contextspace|}$ by  $\meancontext(z)=\Earm{a}[Y|Z=z], \forall z\in\contextspace$. Also, in this section we use $\env_a$ to denote its associated marginal distribution vector $\env_a(Z)\in\cP(\contextspace)\subset\mathbb R^{|\contextspace|}$, and we won't distinguish between an action $a$ and its associated marginal vector $\env_a$.

Recall that in each round $t$ we play some action $A_t$ and then observe context $Z_t$ and reward $Y_t$. By simply ignoring the realized contexts $Z_t$, we can write $Y_t=\sum_{z\in\contextspace}\meancontext(z)\cdot\nu_{A_t}(z)+\actionnoise{t} =\ip{\meancontext, \nu_{A_t}}+\actionnoise{t}$, where $\actionnoise{t}$ is conditionally 1-sub-Gaussian since $\hE[\actionnoise{t}|(A_s, Y_s)_{s\leq t-1}, A_t]=0$ and $\actionnoise{t}\in[-1, 1]$. So now we may think of the game to be linear bandit with actions being $\env_a$ and the unknown mean reward vector being $\meancontext$. Therefore, any linear bandit algorithm that allows such conditionally sub-Gaussian noise condition should be able to operate in our benign setting by ignoring the realized contexts. More importantly, its regret analysis will go through without change, and hence its regret bounds are retained without loss. 

\subsection{Phased Elimination and its Regret Bound}
Among all valid linear bandit algorithms that can be applied in conditionally benign environments, we opt for the phase elimination algorithm ($\elim$) over others due to its superior performance whenever our action set is finite. Its pseudo-code is summarized in \cref{alg: phased-elimination}, which is essentially the same as \citet{lattimore20a}. However, the regret guarantees we present for $\elim$ are novel. Our first result is an anytime worst-case regret bound, which qualifies PE for being a base learner of dynamic balancing. Again, $\elim=(\elim(\delta))_{\delta\in(0,1)}$ is the family of instances of phased elimination algorithm, indexed by the confidence level $\delta$.
\begin{theorem}[Worst-case regret bound for $\elim$]
\label{regret bound: phased elimination w-c eps=0}
    For all $\delta\in(0,1)$, the policy $\elim(\delta)$ given by \cref{alg: phased-elimination} satisfies the following regret bound for all conditionally benign environments $\env$,
    \*[
    \Reg(t)
    \leq
    C\sqrt{\dimspan{\env} \log\lr{\frac{|\actionspace|\log T}{\delta}}t},
    \quad
    \forall t\in[T]
    \]
    with probability at least $1-\delta$, where $\dimspan{\env}=\dim(\mathrm{span}\{\nu_a: a\in\actionspace\})$ and $C>0$ is a universal constant. Note that $\dimspan{\env}\leq|\contextspace|$ and it could be $|\contextspace|$ in the worst-case. In particular, after taking $\delta=1/T$, we obtain the expected regret bound
    \*[
    \hE_{\env,\elim(\delta)}[\Reg(T)]
    =
    O\lr{\sqrt{\dimspan{\env}T\log(|\actionspace|T)}}.
    \]
\end{theorem}
\begin{corollary}
    $\elim$ is a $d$-benign family for $d(\delta)=O\lr{\sqrt{|\contextspace| \log\lr{|\actionspace|\log T/\delta}}}$. Therefore, the expected regret bound in \cref{regret bound: regret balancing general rate} can also be achieved by $\regbalancing$ with $\elim$ and $\ucb$ as base learners.
\end{corollary}
See \cref{appendix sec: proof of upper bounds by PE with approximate marginals} for the proof. Thanks to our reduction, \cref{regret bound: phased elimination w-c eps=0} only depends on $\dimspan{\env}$ (up to log factors) rather than $|\contextspace|$. This indicates that the intrinsic complexity of causal bandit problem is not $|\contextspace|$ and can be further reduced to $\dimspan{\env}$, which is not captured by the $\tilde O(\sqrt{|\contextspace|T})$ regret bound of $\cucb$.

Next we give an instance-dependent regret bound for PE. Notice that this bound is even new for stochastic linear bandits (with finite action sets). See \cref{appendix sec: proof of upper bounds by PE with approximate marginals} for the proof.
\begin{theorem}[Instance-dependent regret bound for $\elim$]\label{regret bound: phased elimination gap-dependent eps=0}
    For all $\delta\in(0,1)$, the policy $\elim(\delta)$ given by \cref{alg: phased-elimination} satisfies the following regret for all conditionally benign environments $\env$,
    \*[
    \Reg(T)
    \leq    C\cdot\frac{\dimspan{\env}\log\lr{|\actionspace|\log T/\delta}}{\mingap(\env)}
    \]
    with probability at least $1-\delta$, where $\mingap(\env):=\min_{a\neq\optarm} \meanaction{\optarm}-\meanaction{a}$ is the minimal sub-optimality gap of instance $\env$ and $C>0$ is a universal constant. 
    In particular, taking $\delta=1/T$,
    \*[
    \hE_{\env,\elim(\delta)}[\Reg(T)]
    =
    O\lr{\frac{\dimspan{\env}\log(|\actionspace|T)}{\mingap(\env)}}.
    \]
\end{theorem}
   \begin{algorithm}[tb]
    \caption{Phased Elimination ($\elim$) in Causal Bandit}\label{alg: phased-elimination}
    \textbf{Input:} Action set $\actionspace$, marginals $\{\nu_a: a\in\actionspace\}$, $\dimspan{\env}=\dim(\mathrm{span}\{\nu_a: a\in\actionspace\})$, and confidence level $\delta\in(0,1)$
    \begin{enumerate}
        \item Set $\ell=1$ and let the initial active set $\activeset{1}$ be $\actionspace$
        \item
        Find some near-optimal design $\phasedesign{\ell}\in\cP(\activeset{\ell})$ with $\max_{a\in\activeset{\ell}} \|\nu_a\|^2_{V(\phasedesign{\ell})^{-1}}\leq 2\dimspan{\env}$ and $|\supp{\phasedesign{\ell}}|\leq 4\dimspan{\env}\log\log(\dimspan{\env})+16$, where $V(\phasedesign{\ell})=\sum_{a\in\activeset{\ell}}\phasedesign{\ell}(a)\nu_a\nu_a^\top$
        \item Let $m_{\ell}=2^{\ell-1}(4\dimspan{\env}\log\log(\dimspan{\env})+16)$. 
        Compute $\phaselength{\ell}(a)=\lrceil{m_\ell\phasedesign{\ell}(a)}$ and $\phaselength{\ell}=\sum_{a\in\activeset{\ell}}\phaselength{\ell}(a)$
        \item Play each action $a\in\activeset{\ell}$ exactly $\phaselength{\ell}(a)$ times and we call these $\phaselength{\ell}$ rounds \textit{phase $\ell$}. We also observe corresponding context-reward pairs $(Z_t, Y_t)_{t\in\phase{\ell}}$
        \item Compute the empirical estimate:
        $\phaseestimate{\ell}=\phasecov{\ell}^{-1}\sum_{t\in\phase{\ell}}\nu_{A_t}Y_t,
        $
        where 
        $\phasecov{\ell}=\sum_{a\in\activeset{\ell}}\phaselength{\ell}(a)\nu_a\nu_a^\top$
        \item Eliminate low rewarding actions and update the active set:\\
        $
        \activeset{\ell+1}=
        \Bigl\{a\in\activeset{\ell}: \max_{b\in\activeset{\ell}} \ip{\phaseestimate{\ell}, \nu_b-\nu_a}
        \leq 2\sqrt{\frac{4\dimspan{\env}}{m_{\ell}}\log\lr{\frac{2|\actionspace|\log_2(T)}{\delta}}}
        \Bigl\}
        $
        \item $\ell\leftarrow\ell+1$ and \textbf{Goto 2}
    \end{enumerate}
\end{algorithm} 
\begin{remark}
   During the implementation of \cref{alg: phased-elimination}, it is possible that $\activeset{\ell}$ cannot span $\mathbb R^{|\contextspace|}$ for some $\ell$ such that $V(\phasedesign{\ell})$ is singular for any $\phasedesign{\ell}\in\cP(\activeset{\ell})$. For example, in later phases $|\activeset{\ell}|$ can be smaller than $|\contextspace|$. Let's say $\dim(\mathrm{span}\{\nu_a: a\in\activeset{\ell}\})=r<|\contextspace|$. One workaround is to apply some invertible matrix $X\in\mathbb R^{|\contextspace|\times|\contextspace|}$ to every $a\in\activeset{\ell}$ such that $X\env_a$ can be decomposed to a dim-$r$ vector $(X\env_a)_{[r]}$ and a tail of $(|\contextspace|-r)$ zeros, and $\lrset{(X\env_a)_{[r]}: a\in\activeset{\ell}}$ can span $\mathbb R^r$. Now we use $\lrset{(X\env_a)_{[r]}: a\in\activeset{\ell}}$ as our active set in phase $\ell$ and the analysis would go through.
\end{remark}

\subsection{Roadblocks: Instance-Dependent Bounds}\label{subsec: adaptive ins-dep regret}
Unlike adaptive worst-case regret studied in \cref{sec:pareto}, adaptive instance-dependent regret is less understood and a general theory is still absent in the literature. In particular, we do not know if $O(\log T)$ regret can always be achieved, and whenever achieved, whether it is tight. These issues are illustrated for model selection methods in the following. First, it is easy to see that $O(\log T)$ regret can always be achieved in benign environments, 
e.g., by corralling $\elim$ and $\ucb$ using dynamic balancing, because in this case both base learners admit logarithmic regret. However, the regret bound of $\ucb$ is dominant and thus naive calculation only leads to a $O(|\actionspace|\log T/\mingap)$ regret for $\regbalancing$. It remains open whether we can adapt to the smaller regret achieved by PE in benign environments. Second, $O(\log T)$ regret is not always granted by model selection in non-benign instances. The only exception we are aware of in the literature is the case where the causal base learner is assumed to incur linear regret whenever its candidate regret bound fails \citep[][Theorem~31]{cutkosky21a}.
If this type of ``algorithm gap'' holds,
$\regbalancing$ will only choose the causal base learner on a $O(\log T)$ number of rounds, and hence enjoy logarithmic regret.
Moreover, without changing the parameter setting, $\regbalancing$ is able to realize the Pareto optimal rates $(\sqrt{|\contextspace|T}, |\actionspace|\sqrt{T}/\sqrt{|\contextspace|})$ up to log factors.
However, the ``algorithmic gap'' requirement on the causal base learner is so stringent that we do not know if it is met by any algorithm in every instance. 
In \cref{appendix sec: proofs regarding adaptive ins-dep results}, we show that a version of $\elim$ incurs linear regret on some instances.

\section{Limited Knowledge of the Marginal Distributions over Context Variables}\label{sec:approx-marginal}
So far, we have assumed that algorithms knows the marginal distribution over the post-action context for each arm. 
Of course, perfect knowledge of these marginals may not hold in practice.
What is the effect of only having access to approximate marginals on achievable rates of regret?

In this section, we study this question.
We give a lower bound indicating that, with zero access to the marginals, 
it is impossible for any algorithm to exploit the causal structure and beat the minimax rate of an arbitrary environment. To model this setting, recall that algorithms are defined as mappings taking $(\actionspace,\contextspace, T, q)$ to policies. So naturally, algorithms considered agnostic to the marginals should be constant in $q\in\cP(\contextspace)^{\actionspace}$, 
leading to the following definition: 
\begin{definition}\label{def: alg without knowledge of marginal}
   An algorithm $\alg$ is said to be \emph{agnostic to marginals} if,
   for any $\actionspace,\contextspace$ and $T$, the map
   \*[
   \alg_{\actionspace,\contextspace,T}:
   q \mapsto \alg(\actionspace,\contextspace,T, q)
   \]
   is constant over $\cP(\contextspace)^{\actionspace}$. We denote the set of all such algorithms by $\noncausalspace$. 
\end{definition}
Examples of algorithms from $\noncausalspace$ include not only heuristic non-causal algorithms like $\ucb$, but also versions of causal algorithms that are always input by the same marginals.
For all algorithm $\alg\in\noncausalspace$, we will write the policy it induces given $\actionspace,\contextspace, T$ as $\alg(\actionspace,\contextspace,T,\cdot)$ to highlight its independence on the $q$ component. Our lower bound shows that, under this zero-marginal-knowledge regime, we cannot do better than the optimal non-causal algorithm.

\begin{theorem}\label{lower bound: when we have zero knowledge of marginals}
    For all $\actionspace, \contextspace, T\geq|\actionspace|$ and MAB algorithms $\alg\in\noncausalspace$, there exists a conditionally benign environment $\env\in\envspace$ such that 
    \*[
    \E_{\env,\alg}[\Reg(T)]
    \geq
    c\sqrt{|\actionspace|T},
    \]
    where $c>0$ is a universal constant.
\end{theorem}
See \cref{appendix sec: proof of lower bound when we have zero knowledge of marginals} for the proof.

\begin{remark}
    Our lower bound  improves on \citet[][Theorem 4]{lu2020regret}, which is of the form of $C_{\varepsilon}\sqrt{|\actionspace|}T^{1/2-\varepsilon}, \forall\varepsilon>0$ and only holds for \emph{some} set of non-causal algorithms, which is a strict subset of $\noncausalspace$. 
\end{remark}

\subsection{Phased Elimination with Approximate Marginals}
Despite the negative result \cref{lower bound: when we have zero knowledge of marginals}, we now argue that some level of misspecification is allowed in the prior knowledge of marginals. Upon interacting with environment $\env$, suppose we are given some marginal $\tilde\env(Z)\in\cP(\contextspace)^{\actionspace}$ which may deviate from the true $\env(Z)$ to some extent. Now we show that even instantiating $\elim$ with the possibly non-accurate $\tilde\env(Z)$ may yield $\tilde O(\sqrt{T})$ regret, following a similar result for $\cucb$ by \citet{bilodeau2022adaptively}. First we need the the following definition to measure the amount of deviation of $\tilde\env(Z)$ from $\env(Z)$.
\begin{definition}\citep[][Definition~4.2]{bilodeau2022adaptively}
    For any $\varepsilon\geq0$, $\tilde\env(Z)$ and $\env(Z)$ are said to be $\varepsilon$-close if
    \*[
    \sup_{a\in\actionspace}\sum_{z\in\contextspace}\lrabs{\tilde\env_a(z)-\env_a(z)}\leq\varepsilon.
    \]
\end{definition}
Due to our reduction in \cref{subsec: causal to linear reduction}, we can find that causal bandits with misspecified marginals is reduced to the well-studied misspecified linear bandits, which yields the following regret bound that subsumes \cref{regret bound: phased elimination w-c eps=0}. The proof is largely based on the analysis of phased elimination in \citet[][Proposition 5.1]{lattimore20a}, with necessary modifications for handling conditionally sub-gaussian noises and providing an anytime regret bound. See \cref{appendix sec: proof of upper bounds by PE with approximate marginals} for details.  
\begin{theorem}[Worst-case regret bound, with approximate marginal distributions]
\label{regret bound: phased elimination w-c general eps}
    In any conditionally environment $\env$ suppose we instantiate $\elim(\delta)$ with $\tilde\env(Z)$. If $\tilde\env(Z)$ and $\env(Z)$ are $\varepsilon-$close, then with probability at least $1-\delta$, the regret of $\elim(\delta)$ is bounded for all rounds $t\in[T]$ by 
    \*[
    \Reg(t)
    \leq
    C\lr{\sqrt{\dimspan{\tilde\env} \log\lr{\frac{|\actionspace|\log T}{\delta}}t}+\varepsilon t\sqrt{\dimspan{\tilde\env}}\log T},
    \]
    where $C>0$ is a universal constant and $\dimspan{\tilde\env}=\dim(\mathrm{span}\{\tilde\env_a: a\in\actionspace\})$
\end{theorem}
It is implied that $\varepsilon=\tilde O(\sqrt{1/T})$ suffices to recover all aforementioned regret guarantees of phased elimination and dynamic balancing. On the other hand, such numerical requirement on $\varepsilon$ is almost necessary for us to avoid the lower bound in \cref{lower bound: when we have zero knowledge of marginals}: from the proof of \cref{lower bound: when we have zero knowledge of marginals} we will find that when $\varepsilon=\Omega(\sqrt{|\actionspace|/T})$, for any algorithm there exists a conditionally benign environment $\env$ and approximate marginal $\tilde\env(Z)$ such that $\tilde\env(Z)$ and $\env(Z)$ are $\varepsilon$-close, but this algorithm would incur $\Omega(\sqrt{|\actionspace|T})$ regret on $\env$ when it is input by $\tilde\env(Z)$.

It is worth mentioning that the $\sqrt{\dimspan{\tilde\env}}$ factor in the misspecification term cannot be improved in many regimes for linear bandit algorithms \citep{lattimore20a}. However, $\cucb$ is able to shave this factor off \citep[][Theorem~4.3]{bilodeau2022adaptively} by utilizing realized contexts rather than the least-square estimate of the mean reward vector $\meancontext$. From this perspective, we see there is a price for pursuing better instance-dependent result by ignoring the context information.

\section{Conclusions and Discussions}
We provide a comprehensive characterization of the Pareto regret frontier for the bandit problem in the context of adapting to causal structure whenever feasible. We also give the first instance-dependent regret bound under conditionally benign environments, based on our novel causal-to-linear reduction. Finally, we show that the common assumption that we have access to the true marginals is necessary in general but still can be relaxed in some cases. 

For future works, it would be important to focus on the design of algorithms that are easier to implement compared to running dynamic balancing over some base learners. On the theoretical side, it would be interesting to investigate other causal bandit scenarios involving adaptivity in light of our Pareto regret frontier. For example, we may define a series of ``semi-benign" settings interpolating conditionally benign environments and non-benign environments and study the Pareto regret frontier thereof. 

\section*{Acknowledgements}
ZL is supported by the Vector Research Grant at the Vector Institute. IA is supported by the Vatat Scholarship from the Israeli Council for Higher Education. DMR is supported by an NSERC Discovery Grant and funding through his Canada CIFAR AI Chair at the Vector Institute. The authors would like to thank Tomer Koren, Blair Bilodeau and Csaba Szepesvári for helpful discussions at different stages of this work.

\bibliography{refs}
\bibliographystyle{icml2024}

\newpage
\appendix
\onecolumn

\section{Regret Analysis for Dynamic Balancing}\label{appendix sec: regret analysis for DB}
In this section we show that the regret guarantees of dynamic balancing in \citet{cutkosky21a} can be generalized to our problem and provide a proof of our main upper bound \cref{regret bound: regret balancing general rate}. 
\paragraph{Notations.} For base learner $i$, we use $\candidreg_i(t)$ to denote its candidate anytime regret bound that is expected to hold in its favorable settings. Throughout we consider $\candidreg_i(t)$ with the form of $d_i\sqrt{t}$, where $d_i$ implicitly depends on the confidence parameter $\delta$. Let $i_t$ be the index of the base learner selected in round $t$. $U_i(t)=\sum_{s=1}^t Y_s\indicator\lrset{i=i_s}$ is the observed cumulative reward in the first $t$ rounds where $i$ is picked, and $n_i(t)=\sum_{s=1}^t \indicator\lrset{i=i_s}$ is the number of rounds $i$ is picked by the end of round $t$. The local regret of $i$ up to round $t$ is $\Reg_i(t)=n_i(t)\meanaction{\optarm}-U_i(t)$. We say learner $i$ is \emph{well-specified} if $\Reg_i(t)\leq\candidreg_i(n_i(t))=d_i\sqrt{n_i(t)}, \forall t\in[T]$ and otherwise it is \emph{misspecified}. We use $i_{\star}$ to denote any well-specified learner.
\subsection{Preliminaries}
Roughly speaking, in each round $t$, dynamic balancing works by (1) running a misspecification test to temporarily de-activate misspecified base learners and (2) picking the learner $i_t$ with minimal putative regret $d_i\sqrt{n_i(t)}$ among all active learners $i$ in this round. In this way, the regret incurred by $\regbalancing$ is comparable to that of the best well-specified learner.

Notice that dynamic balancing was initiated with stochastic contextual bandits (where contexts are revealed \emph{prior to} actions) in \citet{cutkosky21a}. To see that $\regbalancing$ can also be applied in stochastic bandits with post-action contexts, it is worth identifying several important features of $\regbalancing$:
\begin{enumerate}
    \item First of all, the meta decision by $\regbalancing$ on each round $t$ only depends on the global information, i.e. $U_i(t)$ and $n_i(t)$ (as well as user-specified $d_i, b_i$ and $v_i$). In particular, it does not need any information regarding context variables or internal states of base learners.
    \item Second, $\regbalancing$ only updates the selected base learner $i_t$ in each round $t$, and the update only uses the reward and contextual information observed in this round, where the context can be either pre-action or post-action, or both. Thus the regret guarantees of $\regbalancing$ would hold regardless of the nature of contexts given that the internal updates of base learners are not affected.
\end{enumerate}
Therefore, the essence of dynamic balancing does not rely on what kind of (stochastic) contextual information can be observed in the underlying (stochastic) bandit problem due to above observations.

Now we state the worst-case regret bound of $\regbalancing$ in \citet{cutkosky21a} adapted to our setting. First define the good event
\*[
\eventdb(\delta)
=
\lrset{\forall i\in\{1,2\}, \forall t\in T:|n_i(t)\meanaction{\optarm}-U_i(t)-\Reg_i(t)|\leq c\sqrt{n_i(t)\log\lr{\frac{2\log n_i(t)}{\delta}}} }
\] 
on which we are able to control the regret of $\regbalancing$. According to the analysis of \citet[][Lemma 5]{cutkosky21a}, we can fix $c$ to be some absolute constant (which can be actually set to $3$ in our setting) such that $\hP_{\env,\policy}[\eventdb(\delta)]\geq 1-\delta$ for any $\env\in\envspace$ and $\policy\in\policyspace$. Conditioning on $\eventdb(\delta)$, we have the following regret bound:
\begin{proposition}[Adapted version of Theorem 22 in \citet{cutkosky21a}]\label{prop: regret balancing general rate}
    Let $\alg_1$ be a $d_1$-benign family and let $\alg_2$ be a $d_2$-arbitrary family of learners. Let $Z_1, Z_2$ be arbitrary positive real numbers. For all $\delta \in (0,1)$,
    we can set hyper-parameters 
    \*[
    b_i(t)=\max\Bigl\{\frac{2Z_i}{\sqrt{t}},\frac{3\sqrt{2\log(2\log t/\delta)}}{\sqrt{t}} \Bigl\},
    \quad
    v_i=\sqrt{\frac{Z_i}{d_i(\delta)^3}}
    \]
    in dynamic balancing such that, the policy $\regbalancing(\delta)$ given by 
    dynamic balancing with $\alg_1,\alg_2, d_1, d_2$
    satisfies the following:
    for all instances $\env$, 
    conditioning on $\eventdb(\delta)$ and the existence of a well-specified base learner $i_{\star}$, the regret of $\regbalancing(\delta)$ is bounded by 
    \*[
    \Reg(T)
    \leq
    \Reg_{i_{\star}}(T)
    +
    C'\lr{
    \sqrt{\log\lr{\frac{\log T}{\delta}}}
    +
    Z_{i_{\star}}d_{i_{\star}}(\delta)
    +
    \sum_{i\neq i_{\star}} \frac{d_i(\delta)}{Z_i}
    }\sqrt{T},
    \]
    where $C'>0$ is a universal constant.
\end{proposition}
It is straightforward to see that \cref{prop: regret balancing general rate} is obtained by taking $M=2, C=1, c=3, \beta=1/2$, and $W_1=W_2=\sqrt{2}$ in \citet[][Theorem 22]{cutkosky21a}.

\subsection{Proof of \cref{regret bound: regret balancing general rate}}\label{appendix sec: proofs for section 3.1}

\cref{regret bound: regret balancing general rate} is the immediate consequence of the following regret bound, which is derived by instantiating $Z_1, Z_2$ in \cref{prop: regret balancing general rate} with specific values. 
\begin{proposition}\label{prop: regret balancing general rate in terms of R1 R2}
    For every pair of reasonable rate functions $R_1(T), R_2(T)$ such that $R_1(T)R_2(T)\geq |\actionspace|T$, we can instantiate \cref{prop: regret balancing general rate} with $Z_1=1, Z_2=\frac{R_2(T)}{\sqrt{|\actionspace|T}}$
    such that for all $\delta \in (0,1)$,
    the policy $\regbalancing(\delta)$ with the same setup as \cref{prop: regret balancing general rate}
    satisfies the following:
    for all instances $\env$, 
    with probability at least $1-O(\delta)$, 
    the regret of $\regbalancing(\delta)$ is bounded by 
    \*[
    \Reg(T)
    &\leq
    C'\lr{
    d_1(\delta)
    +
    \sqrt{\log\lr{\frac{\log T}{\delta}}}
    +
    \frac{d_2(\delta)}{\sqrt{|\actionspace|T}}R_1(T)
    }\sqrt{T},
    \quad\text{if $\env$ is conditionally benign;}\\
    \Reg(T)
    &\leq
    C'\lr{
    d_1(\delta)
    +
    \sqrt{\log\lr{\frac{\log T}{\delta}}}
    +
    \frac{d_2(\delta)}{\sqrt{|\actionspace|T}}R_2(T)
    }\sqrt{T},
    \quad\text{if $\env$ is arbitrary,}
    \]
    where $C'>0$ is a universal constant.    
\end{proposition}
Now we can see that our main upper bound \cref{regret bound: regret balancing general rate} is proved immediately after taking $d_1(\delta)=O\lr{(\sqrt{|\contextspace|}+\sqrt{\log (T/\delta)})\sqrt{\log(|\contextspace|T/\delta)}}$,
$d_2(\delta)= O(\sqrt{|\actionspace|\log(|\actionspace|T/\delta)})$ and $\delta=1/T$.
\begin{proof}[Proof of \cref{prop: regret balancing general rate in terms of R1 R2}]
    By \cref{def: benign and arbitrary family},
    we know that for all conditionally instances $\env$, with probability at least $1-O(\delta)$, learner $\alg_1$ is well-specified with $\candidreg_1(t)=d_1(\delta)\sqrt{t}$
    and the regret bound in \cref{prop: regret balancing general rate} holds with $i_{\star}=1$.
    Plugging in $Z_1=1, Z_2=\frac{R_2(T)}{\sqrt{|\actionspace|T}}$,
    the regret of $\regbalancing(\delta)$ is bounded by 
    \*[
    \Reg(T)
    \leq
    C'\lr{
    2 d_1(\delta)
    +
    \sqrt{\log\lr{\frac{\log T}{\delta}}}
    +
    d_2(\delta)\frac{\sqrt{|\actionspace|T}}{R_2(T)}
    }\sqrt{T}.
    \]
    Similarly for all instances $\env$, with probability at least $1-O(\delta)$, learner $\alg_2$ is well-specified with $\candidreg_2(t)=d_2(\delta)\sqrt{t}$ and the regret bound in \cref{prop: regret balancing general rate} holds with $i_{\star}=2$, which is
    \*[
    \Reg(T)
    \leq
    C'\lr{
    d_2(\delta)
    +
    \sqrt{\log\lr{\frac{\log T}{\delta}}}
    +
    d_2(\delta)\frac{R_2(T)}{\sqrt{|\actionspace|T}}
    +
    d_1(\delta)
    }\sqrt{T}.
    \]
    By our assumption that $(R_1(T), R_2(T))$ is reasonable and $R_1(T)R_2(T)\geq |\actionspace|T$, we have that $R_2(T)\geq\sqrt{|\actionspace|T}$ and $\frac{|\actionspace|T}{R_2(T)}\leq R_1(T)$. Hence the regret of $\regbalancing(\delta)$ for all instances $\env$ is further bounded by
    \*[
    \Reg(T)
    &\leq
    C'\lr{
    d_1(\delta)
    +
    \sqrt{\log\lr{\frac{\log T}{\delta}}}
    +
    \frac{d_2(\delta)}{\sqrt{|\actionspace|T}}R_1(T)
    }\sqrt{T},
    \quad\text{if $\env$ is conditionally benign;}\\
    \Reg(T)
    &\leq
    C'\lr{
    d_1(\delta)
    +
    \sqrt{\log\lr{\frac{\log T}{\delta}}}
    +
    \frac{d_2(\delta)}{\sqrt{|\actionspace|T}}R_2(T)
    }\sqrt{T},
    \quad\text{if $\env$ is arbitrary,}
    \]
    which completes the proof.
\end{proof}

\section{Regret analysis of phased elimination}\label{appendix sec: proof of upper bounds by PE with approximate marginals}
In this section we will prove \cref{regret bound: phased elimination gap-dependent eps=0} and \cref{regret bound: phased elimination w-c general eps}, while \cref{regret bound: phased elimination w-c eps=0} is implied by taking $\varepsilon=0$ in \cref{regret bound: phased elimination w-c general eps}. Recall that the proof of \cref{regret bound: phased elimination w-c general eps} is based on the analysis of phased elimination in \citet[][Proposition 5.1]{lattimore20a}. For simplicity we will use $\hP$ and $\hE$ to denote the probabilistic operators determined jointly by the underlying conditionally benign environment $\env$ and the phased elimination algorithm. Also we use $\Delta_a, \mingap$ to denote the true sub-optimality gap $\Delta_a(\env)=\meanaction{\optarm}-\meanaction{a}$ and minimal sub-optimality gap $\mingap(\env)=\min_{a\in\actionspace} \meanaction{\optarm}-\meanaction{a}$ respectively with regards to the underlying instance $\env$. 
\subsection{Prerequisite}
\begin{lemma}\label{fact: in-phase concentration}(In-phase concentration)
    For any phase $\ell$, let
    \*[
    \phaseconcentration{\ell}(\delta)=\lrset{|\ip{\phaseestimate{\ell}-\meancontext, \tilde\nu_a}|\leq 2\varepsilon\sqrt{\dimspan{\tilde\env}}+\sqrt{\frac{4\dimspan{\tilde\env}}{m_\ell}\log\lr{\frac{2|\actionspace|\log_2(T)}{\delta}}}, \forall a\in\activeset{\ell}}
    \]
    and $\mathscr{F}_{\ell}$ be the $\sigma-$algebra generated by the history up to the start of phase $\ell$. Then $\hP[\phaseconcentration{\ell}(\delta)|\mathscr{F}_{\ell}]\geq 1-\frac{\delta}{\log_2(T)}$.
\end{lemma}
\begin{proof}[Proof of \cref{fact: in-phase concentration}]
    Let $b_a=\ip{\nu_a-\tilde\nu_a, \meancontext}, \forall a\in\actionspace$ be the error term due to the use of inaccurate marginals, then we know that $|b_a|\leq\varepsilon, \forall a\in\actionspace$ since $\env(Z)$ and $\tilde\env(Z)$ are $\varepsilon-$close.
    Observe that 
    \*[
    \ip{\phaseestimate{\ell}-\meancontext, \tilde\nu_a}
    &=\ip{\phasecov{\ell}^{-1}\sum_{t\in\phase{\ell}}\tilde\nu_{A_t}\tilde\nu_{A_t}^\top \meancontext, \tilde\nu_a }
    -\ip{\meancontext, \tilde\nu_a} \\
    &+
    \ip{\phasecov{\ell}^{-1}\sum_{t\in\phase{\ell}}\tilde\nu_{A_t}\actionnoise{t}, \tilde\nu_a}
    +
    \ip{\phasecov{\ell}^{-1}\sum_{t\in\phase{\ell}}\tilde\nu_{A_t}b_{A_t}, \tilde\nu_a}\\
    &=
    \sum_{t\in\phase{\ell}}\ip{\phasecov{\ell}^{-1}\tilde\nu_{A_t}, \tilde\nu_a}\actionnoise{t}
    +
    \sum_{t\in\phase{\ell}}\ip{\phasecov{\ell}^{-1}\tilde\nu_{A_t}, \tilde\nu_a}b_{A_t}.
    \]
    Using Cauchy-Schwarz inequality and the fact that for all $a\in\activeset{\ell}, \|\tilde\nu_a\|_{\phasecov{\ell}^{-1}}^2\leq\frac{1}{m_{\ell}}\|\tilde\nu_a\|_{V(\pi_{\ell})^{-1}}^2\leq\frac{2\dimspan{\tilde\env}}{m_{\ell}}$, the second term on the RHS of the above equality can be bounded by 
    \*[
    \lrabs{\sum_{t\in\phase{\ell}}\ip{\phasecov{\ell}^{-1}\tilde\nu_{A_t}, \tilde\nu_a}b_{A_t}}
    &\leq
    \varepsilon\sum_{t\in\phase{\ell}}\lrabs{\ip{\phasecov{\ell}^{-1}\tilde\nu_{A_t},\tilde\nu_a}} \\
    &\leq
    \varepsilon\sqrt{\lr{\sum_{t\in\phase{\ell}}1} \lr{\sum_{t\in\phase{\ell}} \ip{\phasecov{\ell}^{-1}\tilde\nu_{A_t},\tilde\nu_a}^2}}\\
    &=
    \varepsilon\sqrt{\phaselength{\ell}\|\tilde\nu_a\|_{\phasecov{\ell}^{-1}}^2}
    \leq
    \varepsilon\sqrt{2m_{\ell}\frac{2\dimspan{\tilde\env}}{m_{\ell}}}
    =2\varepsilon\sqrt{\dimspan{\tilde\env}}.
    \]
    To bound the first term, notice that $(A_t)_{t\in\phase{\ell}}, \phasecov{\ell}$ are fixed given the history prior to the start of phase $\ell$. Hence $(\actionnoise{t})_{t\in\phase{\ell}}$ are independent conditioned on $\mathscr F_{\ell}$ and bounded by $[-1,1]$. By standard concentration bounds, we have that with probability at least $1-\frac{\delta}{|\actionspace|\log_2(T)}$,
    \*[
    \lrabs{\sum_{t\in\phase{\ell}}\ip{\phasecov{\ell}^{-1}\nu_{A_t}, \tilde\nu_a}\actionnoise{t}}
    &\leq
    \sqrt{2\sum_{t\in\phase{\ell}}\ip{\phasecov{\ell}^{-1}\tilde\nu_{A_t}, \tilde\nu_a}^2\log\lr{\frac{2|\actionspace|\log_2(T)}{\delta}} },
    \] 
    where the RHS can be rewritten as
    \*[
    \sqrt{2\|\tilde\nu_a\|_{\phasecov{\ell}^{-1}}^2 \log\lr{\frac{2|\actionspace|\log_2(T)}{\delta}} }
    \leq
    \sqrt{\frac{4\dimspan{\tilde\env}}{m_{\ell}} \log\lr{\frac{2|\actionspace|\log_2(T)}{\delta}} }.
    \] 
    Combining the two upper bounds above and taking a union bound over all $a\in\activeset{\ell}$, we have that with probability at least $1-\frac{\delta}{\log_2(T)}$,
    \*[
    |\ip{\phaseestimate{\ell}-\meancontext, \tilde\nu_a}|\leq 2\varepsilon\sqrt{\dimspan{\tilde\env}}+\sqrt{\frac{4\dimspan{\tilde\env}}{m_\ell}\log\lr{\frac{2|\actionspace|\log_2(T)}{\delta}}}, 
    \quad
    \forall a\in\activeset{\ell},
    \]
    which finishes the proof.
\end{proof}

Since the marginal distributions $\tilde\nu_a$ are possibly not accurate, we may not be able to show that the optimal action $\optarm$ is never eliminated with high probability. So what we can hope for is that actions that are near-optimal \textit{relative to} the best action in $\activeset{\ell}$ are retained in the end of the phase $\ell$. To be concrete, define $\trueoptarm{\ell}\in\cArgmin_{a\in\activeset{\ell}}\Delta_a$ to be the true optimal action within $\activeset{\ell}$. Then we can show that $\Delta_a-\Delta_{\trueoptarm{\ell}}$ is rather small for any $a$ that is not eliminated in the end of phase $\ell$. 
\begin{lemma}\label{fact: relative suboptimality of arms that are not eliminated} 
    Conditioning on event $\phaseconcentration{\ell}(\delta)$, for any action $a$ not eliminated in the end of phase $\ell$, it has relative sub-optimality gap $\ip{\meancontext, \nu_{\trueoptarm{\ell}}-\nu_a}=\Delta_a-\Delta_{\trueoptarm{\ell}}\leq 2\varepsilon(1+2\sqrt{\dimspan{\tilde\env}})+4\sqrt{\frac{4\dimspan{\tilde\env}}{m_\ell}\log\lr{\frac{2|\actionspace|\log_2(T)}{\delta}}}$.
\end{lemma}
\begin{proof}[Proof of \cref{fact: relative suboptimality of arms that are not eliminated}]
    According to the rule of updating active set, whenever $a\in\activeset{\ell}$ is not eliminated at the end of phase $\ell$, it holds
    \*[
    \ip{\phaseestimate{\ell}, \tilde\nu_{\trueoptarm{\ell}}-\tilde\nu_a}
    \leq
    \max_{b\in\activeset{\ell}}\ip{\phaseestimate{\ell}, \tilde\nu_b-\tilde\nu_a}
    \leq
    2\sqrt{\frac{4\dimspan{\tilde\env}}{m_\ell}\log\lr{\frac{2|\actionspace|\log_2(T)}{\delta}}}.
    \]
    It implies that
    \*[
    \ip{\meancontext, \tilde\nu_{\trueoptarm{\ell}}-\tilde\nu_a} 
    &=
    \ip{\meancontext-\phaseestimate{\ell}, \tilde\nu_{\trueoptarm{\ell}}-\tilde\nu_a}
    +
    \ip{\phaseestimate{\ell}, \tilde\nu_{\trueoptarm{\ell}}-\tilde\nu_a} \\
    &\leq
    2\LR{
    \sqrt{\frac{4\dimspan{\tilde\env}}{m_\ell}\log\lr{\frac{2|\actionspace|\log_2(T)}{\delta}}}
    +
    2\varepsilon\sqrt{\dimspan{\tilde\env}}
    }
    +
    2\sqrt{\frac{4\dimspan{\tilde\env}}{m_\ell}\log\lr{\frac{2|\actionspace|\log_2(T)}{\delta}}} \\
    &=
    4\sqrt{\frac{4\dimspan{\tilde\env}}{m_\ell}\log\lr{\frac{2|\actionspace|\log_2(T)}{\delta}}}
    + 4\varepsilon\sqrt{\dimspan{\tilde\env}}.
    \]
    where we use the fact that we are conditioning on $\phaseconcentration{\ell}(\delta)$ in the inequality. Hence under the true marginals $\env$,
    \*[
    \ip{\meancontext, \nu_{\trueoptarm{\ell}}-\nu_a} 
    \leq
    4\sqrt{\frac{4\dimspan{\tilde\env}}{m_\ell}\log\lr{\frac{2|\actionspace|\log_2(T)}{\delta}}}
    + 
    4\varepsilon\sqrt{\dimspan{\tilde\env}}
    + 
    2\varepsilon.
    \]
\end{proof}

Now we need to track $\Delta_{\trueoptarm{\ell}}$, the sub-optimality of the best active action in each phase. Observe that $\Delta_{\trueoptarm{\ell}}=\sum_{k=1}^{\ell-1} (\Delta_{\trueoptarm{k+1}}-\Delta_{\trueoptarm{k}})$ since $\Delta_{\trueoptarm{1}}=\Delta_{\optarm}=0$. Then it suffices to control each $\Delta_{\trueoptarm{k+1}}-\Delta_{\trueoptarm{k}}, k\leq\ell-1$, to control the growth of $\Delta_{\trueoptarm{\ell}}$.
\begin{lemma}\label{fact: suboptimality of best arm in each phase} 
    Conditioning on event $\phaseconcentration{\ell}(\delta)$, we have $\Delta_{\trueoptarm{\ell+1}}-\Delta_{\trueoptarm{\ell}}\leq 2\varepsilon(1+2\sqrt{\dimspan{\env}})$. 
\end{lemma}
\begin{proof}[Proof of \cref{fact: suboptimality of best arm in each phase}]
    Suppose $\phaseconcentration{\ell}(\delta)$ happens. Notice that the results holds trivially if $\trueoptarm{\ell}$ is not eliminated in the end of phase $\ell$, because in this case $\trueoptarm{\ell+1}=\trueoptarm{\ell}$. On the other hand, if $\trueoptarm{\ell}$ is eliminated, define $\empoptarm{\ell}\in\cArgmax_{a\in\activeset{\ell}}\ip{\phaseestimate{\ell}, \tilde\nu_a}$ to be the empirically best action in the end of phase $\ell$ and then we have
    \*[
    \ip{\empmeancontext{\ell}, \tilde\nu_{\empoptarm{\ell}}-\tilde\nu_{\trueoptarm{\ell}}}
    >
    2\sqrt{\frac{4\dimspan{\tilde\env}}{m_{\ell}}\log\lr{\frac{2|\actionspace|\log_2(T)}{\delta}}},
    \]
    according to the test performed. In the meantime, recall that due to in-phase concentration and $\varepsilon-$closeness between $\tilde\env$ and $\env$,
    \*[
    \ip{\empmeancontext{\ell}, \tilde\nu_{\empoptarm{\ell}}-\tilde\nu_{\trueoptarm{\ell}}}
    &\leq
    \ip{\meancontext, \tilde\nu_{\empoptarm{\ell}}-\tilde\nu_{\trueoptarm{\ell}}}
    +4\varepsilon\sqrt{\dimspan{\tilde\env}}
    +2\sqrt{\frac{4\dimspan{\tilde\env}}{m_{\ell}}\log\lr{\frac{2|\actionspace|\log_2(T)}{\delta}}} \\
    &\leq
    \ip{\meancontext, \nu_{\empoptarm{\ell}}-\nu_{\trueoptarm{\ell}}} + 2\varepsilon
    +4\varepsilon\sqrt{\dimspan{\tilde\env}}
    +2\sqrt{\frac{4\dimspan{\tilde\env}}{m_{\ell}}\log\lr{\frac{2|\actionspace|\log_2(T)}{\delta}}}.
    \]
    Hence we get
    \*[
    \Delta_{\empoptarm{\ell}}-\Delta_{\trueoptarm{\ell}}
    =
    \ip{\meancontext, \nu_{\trueoptarm{\ell}}-\nu_{\empoptarm{\ell}}}
    \leq
    2\varepsilon
    +4\varepsilon\sqrt{\dimspan{\tilde\env}}
    \]
    and 
    \*[
    \Delta_{\trueoptarm{\ell+1}}-\Delta_{\trueoptarm{\ell}}
    \leq 
    \Delta_{\empoptarm{\ell}}-\Delta_{\trueoptarm{\ell}}
    \leq
    2\varepsilon
    +4\varepsilon\sqrt{\dimspan{\tilde\env}}.
    \]
\end{proof}
\begin{corollary}\label{fact: suboptiamlity of all active actions}
    For any $\ell\geq 2$ and conditioning on $\bigcap_{k\leq\ell-1}\phaseconcentration{k}(\delta)$, we have that $\Delta_{\trueoptarm{\ell}}\leq 2\varepsilon(\ell-1)(1+2\sqrt{\dimspan{\tilde\env}})$ and $\Delta_a\leq 2\varepsilon(1+2\sqrt{\dimspan{\tilde\env}})+4\sqrt{\frac{4\dimspan{\tilde\env}}{m_{\ell-1}}\log\lr{\frac{2|\actionspace|\log_2(T)}{\delta}}}+2\varepsilon(\ell-2)(1+2\sqrt{\dimspan{\tilde\env}})$ for all $a\in\activeset{\ell}$. 
\end{corollary}
\begin{proof}[Proof of \cref{fact: suboptiamlity of all active actions}]
    By conditioning on the intersection of all $\phaseconcentration{k}(\delta), k\leq\ell-1$, we have that
    \*[
    \Delta_{\trueoptarm{k+1}}-\Delta_{\trueoptarm{k}}
    \leq 
    2\varepsilon(1+2\sqrt{\dimspan{\tilde\env}}),
    \forall k\leq \ell-1,
    \]
    which implies that $\Delta_{\trueoptarm{s}}=\sum_{k=1}^{s-1} (\Delta_{\trueoptarm{k+1}}-\Delta_{\trueoptarm{k}})\leq 2\varepsilon(s-1)(1+2\sqrt{\dimspan{\tilde\env}}), \forall s\leq\ell$. In particular, there is
    \*[
    \Delta_{\trueoptarm{\ell}}
    \leq
    2\varepsilon(\ell-1)(1+2\sqrt{\dimspan{\tilde\env}}).
    \] 
    Since every action $a\in\activeset{\ell}$ passes the test in the end of $(\ell-1)-$th phase and hence is not eliminated, by \cref{fact: relative suboptimality of arms that are not eliminated} we know 
    \*[
    \Delta_a-\Delta_{\trueoptarm{\ell-1}}
    \leq 2\varepsilon(1+2\sqrt{\dimspan{\tilde\env}})
    +
    4\sqrt{\frac{4\dimspan{\tilde\env}}{m_{\ell-1}}\log\lr{\frac{2|\actionspace|\log_2(T)}{\delta}}}.
    \]
    Therefore, for all $a\in\activeset{\ell}$, 
    \*[
    \Delta_a
    =
    \Delta_a-\Delta_{\trueoptarm{\ell-1}}
    +
    \Delta_{\trueoptarm{\ell-1}}
    \leq 2\varepsilon(1+2\sqrt{\dimspan{\tilde\env}})
    +
    4\sqrt{\frac{4\dimspan{\tilde\env}}{m_{\ell-1}}\log\lr{\frac{2|\actionspace|\log_2(T)}{\delta}}}+2\varepsilon(\ell-2)(1+2\sqrt{\dimspan{\tilde\env}}).
    \]
\end{proof}
\subsection{Proof of \cref{regret bound: phased elimination w-c general eps}}
Now we are prepared to prove \cref{regret bound: phased elimination w-c general eps}.
\begin{proof}[Proof of \cref{regret bound: phased elimination w-c general eps}]
    Let $\lastphase(t)$ be the index of the phase where round $t$ is located. It's easy to see that $\lastphase(T)\leq\log_2(T)$. In the following we condition on the event $\bigcap_{\ell\leq\lastphase(T)} \phaseconcentration{\ell}(\delta)$, which happens with probability at least $1-\delta$ due to \cref{fact: in-phase concentration}.

    Notice that phase $\lastphase(t)$ is not necessarily completed in the end of round $t$, but we can always round $\Reg(t)$ to the regret incurred in the first $\lastphase(t)$ complete phases. That is, 
    \*[
    \Reg(t)\leq\sum_{\ell=1}^{\lastphase(t)}\sum_{a\in\activeset{\ell}}\phaselength{\ell}(a)\cdot\Delta_a.
    \]
    Since we have controlled sub-optimality of all active actions in \cref{fact: suboptiamlity of all active actions}, it holds with probability at least $1-\delta$ that
    \*[
     \Reg(t)
     &\leq
     \sum_{\ell=1}^{\lastphase(t)}\sum_{a\in\activeset{\ell}}\phaselength{\ell}(a)\cdot\Delta_a \\
     &\leq
     2m_1 
     + 
     C\sum_{\ell=2}^{\lastphase(t)} m_{\ell}\lr{\sqrt{\frac{\dimspan{\tilde\env}}{m_{\ell-1}}\log\lr{\frac{2|\actionspace|\log_2(T)}{\delta}}}+\varepsilon\ell\sqrt{\dimspan{\tilde\env}}} \\
     &\leq
     2m_1
     +
     C\sum_{\ell=2}^{\lastphase(t)}\sqrt{m_{\ell}\cdot\dimspan{\tilde\env}\cdot\log\lr{\frac{2|\actionspace|\log_2(T)}{\delta}}}
     +
     C\varepsilon\sqrt{\dimspan{\tilde\env}}\sum_{\ell=2}^{\lastphase(t)} m_{\ell}\ell
     \\
     &\leq
     2m_1 
     +
     C\sqrt{m_{\lastphase(t)}\cdot\dimspan{\tilde\env}\cdot\log\lr{\frac{2|\actionspace|\log_2(T)}{\delta}}}
     + C\varepsilon\sqrt{\dimspan{\tilde\env}}m_{\lastphase(t)}\log_2(T) \\
     &\leq
     C\lr{\sqrt{\dimspan{\tilde\env} t\log\lr{\frac{2|\actionspace|\log T}{\delta}}}+\varepsilon t\sqrt{\dimspan{\tilde\env}}\log T},
    \]
    where $C>0$ is an absolute constant that can vary from line to line. Thus we have finished the proof.
\end{proof}

\subsection{Proof of \cref{regret bound: phased elimination gap-dependent eps=0}}
Now we go back to the setting where $\varepsilon=0$. The only modification needed to work out \cref{regret bound: phased elimination gap-dependent eps=0} is an instance-dependent control over the number of phases for which sub-optimal arms are not entirely eliminated.

\begin{proof}[Proof of \cref{regret bound: phased elimination gap-dependent eps=0}]
    Again suppose $\phaseconcentration{\ell}(\delta)$ happens for all $\ell$. From \cref{fact: suboptiamlity of all active actions} we know that every suboptimal action $a$ can only be played in those phase $\ell\geq2$ s.t. $\Delta_a\leq4\sqrt{\frac{4\dimspan{\env}}{m_{\ell-1}}\log\lr{\frac{2|\actionspace|\log_2(T)}{\delta}}}$ in addition to the first phase. Let
    \*[
    \ell_a
    =
    \max\LRset{
    \ell\geq2: \Delta_a\leq4\sqrt{\frac{4\dimspan{\env}}{m_{\ell-1}}\log\lr{\frac{2|\actionspace|\log_2(T)}{\delta}}} }
    \]
    be the maximal number of phases where $a$ can be played. It is easy to see that
    \*[
    \ell_a
    =
    2+\lrfloor{\log_2\lr{\frac{64\dimspan{\env}}{m_1\Delta_a^2}\log\lr{\frac{2|\actionspace|\log_2(T)}{\delta}}}}.
    \]
    Hence there are at most $\lastphase=2+\lrfloor{\log_2\lr{\frac{64\dimspan{\env}}{m_1\mingap^2}\log\lr{\frac{2|\actionspace|\log_2(T)}{\delta}}}}$ number of phases before all suboptimals are eliminated and $\Reg(T)$ can be controlled more carefully:
    \*[
    \Reg(T)
    &\leq
    \sum_{\ell=1}^{\lastphase}\sum_{a\in\activeset{\ell}}\phaselength{\ell}(a)\cdot\Delta_a \\
    &\leq
    2m_1 + C\sqrt{m_{\lastphase}\cdot\dimspan{\env}\cdot\log\lr{\frac{2|\actionspace|\log_2(T)}{\delta}}} \\
    &=
    2m_1+C\sqrt{2^{\lastphase}\cdot m_1\cdot\dimspan{\env}\cdot\log\lr{\frac{2|\actionspace|\log_2(T)}{\delta}}} \\
    &\leq
    2m_1+C\sqrt{\frac{\dimspan{\env}\log\lr{\frac{2|\actionspace|\log_2(T)}{\delta}}}{m_1\mingap^2}\cdot m_1\cdot\dimspan{\env}\log\lr{\frac{2|\actionspace|\log_2(T)}{\delta}}}\\
    &\leq
    C\cdot\frac{\dimspan{\env}\log\lr{|\actionspace|\log T/\delta}}{\mingap},
    \]
    where $C>0$ is an absolute constant that can vary from line to line. Again the above regret bound holds with probability at least $1-\delta$ so we are done.
\end{proof}

\section{Anytime Regret Bounds for UCB and C-UCB}\label{appendix sec: anytime regret for ucb and cucb}
In this section we verify \cref{regret bound: cucb and ucb} for $\ucb$ and $\cucb$ algorithms for completeness. Note that our anytime regret bound for $\cucb$ is new in the literature.
\subsection{Preliminaries}
For each $t\in[T], a\in\actionspace$ and $z\in\contextspace$, define $\numaction{t}(a)=1 \vee \sum_{s=1}^t \indicator\{A_s=a \}$ to be the number of action $a$ being chosen in the first $t$ rounds, and define $\numcontext{t}(z)=1 \vee \sum_{s=1}^t\indicator\{Z_s(A_s)=z\}$ to be the number of context $z$ being observed up to the first $t$ rounds. Further define the mean reward estimates $\empmeanaction{t}(a), \empmeancontext{t}(z)$ by
\*[
\empmeanaction{t}(a)
&=
\frac{1}{\numaction{t}(a)}
\sum_{s=1}^t Y_s(A_s)\indicator\{A_s=a\} \\
\empmeancontext{t}(z)
&=
\frac{1}{\numcontext{t}(z)}
\sum_{s=1}^t Y_s(A_s)\indicator\{Z_s(A_s)=z\}
\]
Then we introduce the upper confidence bounds used by the UCB-type algorithms under consideration. Given any prescribed confidence parameter $\delta\in(0,1)$, define $\ucbaction{t}(a)
=
\empmeanaction{t}(a)+\sqrt{\frac{\log(2|\actionspace|T/\delta)
}{2\numaction{t}(a)}}, \ucbcontext{t}(z)=\empmeancontext{t}(z)+\sqrt{\frac{\log(2|\contextspace|T/\delta)}{2\numcontext{t}(z)}}$ and $\cucbaction{t}(a)=\sum_{z\in\contextspace}\ucbcontext{t}(z)\Parm{a}[Z=z]$ for each $t\in[T], a\in\actionspace$ and $z\in\contextspace$. Furthermore, we use $\ucb(\delta)$ and $\cucb(\delta)$ to denote the standard UCB algorithm and C-UCB algorithm (\citealt{lu2020regret}) which run by playing actions $\ucbactionrv{t}$ and $\cucbactionrv{t}$ at each round $t$ respectively, according to:
\*[
\ucbactionrv{t}
&=
\cArgmax_{a\in\actionspace}\ucbaction{t-1}(a)  \\
\cucbactionrv{t}
&=
\cArgmax_{a\in\actionspace}
\cucbaction{t-1}(a).
\]
Before analyzing the regret of $\ucb(\delta)$ and $\cucb(\delta)$, let's finally define some high-probability events on which we can control the regret. For any given confidence parameters $\delta, \delta'$, define
\*[
\eventaction(\delta)
=
\bigg\{\forall t\in[T], a\in\actionspace, |\empmeanaction{t}(a)-\meanaction{a}|\leq \sqrt{\frac{\log(2|\actionspace|T/\delta)}{2\numaction{t}(a)}} \bigg\},
\]
and in conditionally benign environments we additionally define
\*[
\eventcontext(\delta)
&=
\bigg\{\forall t\in[T], z\in\contextspace, |\empmeancontext{t}(z)-\meancontext(z))|\leq \sqrt{\frac{\log(2|\contextspace|T/\delta)}{2\numcontext{t}(z)}} \bigg\}, \\
\eventmartingale(\delta')
&=
\bigg\{\forall t\in[T], \sum_{s=1}^t\sum_{z\in\contextspace}\frac{1}{\sqrt{\numcontext{s-1}(z)}}(\Parm{A_s}[Z=z]-\indicator\{Z_s=z\}) \leq \sqrt{2t\log(T/\delta')} \bigg\},
\]
where we recall $\meancontext(z) =\Earm{a}[Y|Z=z]$ is well-defined here. First we can see that $\eventaction(\delta)$ and $\eventcontext(\delta)$ happen with probability at least $1-\delta$ regardless the underlying environment and chosen policy:
\begin{lemma}[Lemma B.1 and B.2 in {\citealt{bilodeau2022adaptively}}]
\label{high prob event: action and context}
For any $\env\in\envspace$ and $\policy\in\policyspace$, 
    \*[
    \Pnp[(\eventaction(\delta))^c] 
    \leq 
    \delta,
    \]
    and for any $\env\in\envspace$ that is conditionally benign and $\policy\in\policyspace$,
    \*[
    \Pnp[(\eventcontext(\delta))^c]
    \leq \delta.
    \]
\end{lemma}
To get our new anytime regret bound for $\cucb(\delta)$, we need to further condition on $\eventmartingale(\delta')$ which happens with probability at least $1-\delta'$:
\begin{lemma}\label{high prob event: martingale}
For any $\env\in\envspace$ and $\policy\in\policyspace$,
\*[
\Pnp[(\eventmartingale(\delta'))^c]
\leq \delta'
\]
\end{lemma}
\begin{proof}[Proof of \cref{high prob event: martingale}]
    Define 
    \*[
    M_t 
    &= \sum_{s=1}^t\sum_{z\in\contextspace}\frac{1}{\sqrt{\numcontext{s-1}(z)}}(\Parm{A_s}[Z=z]-\indicator\{Z_s=z\}), \forall t\in[T],\\
    M_0
    &=0.
    \]
    Then $\eventmartingale(\delta')=\bigg\{\forall t\in[T], M_t\leq \sqrt{2t\log(T/\delta')} \bigg\}$ and it is easy to find that $\{M_t\}_{t\geq 0}$ is a martingale sequence with respect to $\mathcal F_t=\sigma(A_t, H_{t-1})$. To see this, 
    \*[
    \hE_{\env, \policy}[M_t|A_t, H_{t-1}]
    =
    M_{t-1}
    +
    \sum_{z\in\contextspace}\frac{1}{\sqrt{\numcontext{t-1}(z)}}\hE_{\env,\policy}[\hP_{\env_{A_t}}[Z=z]-\indicator\{Z_t=z\}|A_t]
    =
    M_{t-1}.
    \]
    Also, 
    \*[
    |M_t-M_{t-1}|
    &=
    \lrabs{
    \sum_{z\in\contextspace}\frac{1}{\sqrt{\numcontext{t-1}(z)}}
    \lr{
    \Parm{A_t}[Z=z]-\indicator\{Z_t=z\}
    }
    } \\
    &=
    \lrabs{
    \Enp
    \LRbra{
    \sum_{z\in\contextspace}\frac{1}{\sqrt{\numcontext{t-1}(z)}}\indicator\{Z_t=z\}|A_t, H_{t-1}
    }
    -
    \sum_{z\in\contextspace}\frac{1}{\sqrt{\numcontext{t-1}(z)}}\indicator\{Z_t=z\}
    } \\
    &=
    \lrabs{
    \Enp
    \LRbra{
    \frac{1}{\sqrt{\numcontext{t-1}(Z_t)}}|A_t, H_{t-1}
    }
    -
    \frac{1}{\sqrt{\numcontext{t-1}(Z_t)}}
    } \\
    &\leq
    1.
    \]
    Then by Azuma-Hoeffding,
    \*[
    \Pnp[M_t>\sqrt{2t\log(T/\delta')}]
    &=
    \Pnp[M_t-M_0>\sqrt{2t\log(T/\delta')}] \\
    &\leq
    \exp\lr{-\frac{2t\log(T/\delta')}{2t}}=\delta'/T,
    \]
    and we get $\Pnp[(\eventmartingale(\delta'))^c]\leq \delta'$ after taking a union bound over $t\in[T]$.
\end{proof}
\subsection{Anytime High-probability Regret Bound}

Now we provide our high-probability regret bounds for $\ucb(\delta)$ and $\cucb(\delta)$ that will lead to \cref{regret bound: cucb and ucb}.
\begin{theorem}\label{regret bound: UCB} In any environment $\env$, the regret of $\ucb(\delta)$ is bounded by
\*[
\Reg(t) = O\lr{\sqrt{|\actionspace|\log(|\actionspace|T/\delta) t}}
\] 
for all $t\in[T]$, conditioning on event $\eventaction(\delta)$ which happens with probability at least $1-\delta$.    
\end{theorem}
\begin{proof}[Proof of \cref{regret bound: UCB}]
    In event $\eventaction(\delta)$, we have that $\meanaction{a}\leq \ucbaction{t}(a)\leq \meanaction{a}+2\sqrt{\frac{\log(2|\actionspace|T/\delta)}{2\numaction{t}(a)}}$ for all $a\in\actionspace, t\in[T]$. Hence conditioned on $\eventaction(\delta)$, the regret of $\ucb(\delta)$ up to any round $t\in[T]$ holds
    \*[
    \Reg(t)
    &=
    \sum_{s=1}^t \meanaction{\optarm}-\meanaction{A_s} \\
    &=
    \sum_{s=1}^t (\meanaction{\optarm}-\ucbaction{s-1}(A_s)) + (\ucbaction{s-1}(A_s)-\meanaction{A_s}) \\
    &\leq
    \sum_{s=1}^t (\ucbaction{s-1}(\optarm)-\ucbaction{s-1}(A_s)) + (\ucbaction{s-1}(A_s)-\meanaction{A_s}) \\
    &\leq
    \sum_{s=1}^t (\ucbaction{s-1}(A_s)-\meanaction{A_s})\\
    &\leq
    \sum_{s=1}^t \sqrt{\frac{2\log(2|\actionspace|T/\delta)}{\numaction{s-1}(A_s)}}\\
    &=
    \sum_{s=1}^t\sum_{a\in\actionspace} \sqrt{\frac{2\log(2|\actionspace|T/\delta)}{\numaction{s-1}(A_s)}}\indicator\{A_s=a\} \\
    &\leq
    \sum_{a\in\actionspace} \sqrt{8\log(2|\actionspace|T/\delta)\numaction{t-1}(a)}\\
    &\leq
    \sqrt{8\log(2|\actionspace|T/\delta)|\actionspace|t},
    \]
    where we use $A_s=\ucbactionrv{s}$ throughout to simplify our notation.
\end{proof}

\begin{theorem}\label{regret bound: C-UCB} In any conditionally benign environment $\env$, the regret of $\cucb(\delta)$ is bounded by
\*[
\Reg(t) = O \lr{\sqrt{\log(|\contextspace|T/\delta)}\lr{\sqrt{|\contextspace|}+\sqrt{\log(T/\delta')}}\sqrt{t}}
\]
for all $t\in[T]$, conditioning on event $\eventcontext(\delta)\cap\eventmartingale(\delta')$ which happens with probability at least $1-\delta-\delta'$.   
\end{theorem}
\begin{proof}[Proof of \cref{regret bound: C-UCB}]
    Similarly in event $\eventcontext(\delta)$ we have $\meancontext(z)\leq\ucbcontext{t}(z)\leq\meancontext(z)+2\sqrt{\frac{\log(2|\contextspace|T/\delta)}{2\numcontext{t}(z)}}$ for all $z\in\contextspace, t\in[T]$. Additionally, 
    \*[
    \meanaction{\optarm}
    &=
    \sum_{z\in\contextspace} \meancontext(z)\Parm{\optarm}[Z=z]\\
    &\leq
    \sum_{z\in\contextspace} \ucbcontext{t-1}(z)\Parm{\optarm}[Z=z]\\
    &=\cucbaction{t-1}(\optarm)
    \leq \cucbaction{t-1}(A_t), \forall t\in[T],
    \]
    where $A_t=\cucbactionrv{t}$ is the action played by $\cucb(\delta)$.
    Therefore we can control the cumulative the regret of $\cucb(\delta)$ in the first $t$ rounds as follows
    \*[
    \Reg(t)
    &=
    \sum_{s=1}^t (\meanaction{\optarm}-\cucbaction{s-1}(A_s))+(\cucbaction{s-1}(A_s)-\meanaction{A_s})\\
    &\leq
    \sum_{s=1}^t \cucbaction{s-1}(A_s)-\meanaction{A_s}\\
    &=
    \sum_{s=1}^t\sum_{z\in\contextspace} (\ucbcontext{s-1}(z)-\meancontext(z))\Parm{A_s}[Z=z]\\
    &\leq
    \sum_{s=1}^t\sum_{z\in\contextspace} \sqrt{\frac{2\log(2|\contextspace|T/\delta)}{\numcontext{s-1}(z)}}\Parm{A_s}[Z=z]\\
    &=
    \sum_{s=1}^t\sum_{z\in\contextspace}
    \sqrt{\frac{2\log(2|\contextspace|T/\delta)}{\numcontext{s-1}(z)}}\indicator\{Z_s=z\}
    +\sum_{s=1}^t\sum_{z\in\contextspace}
    \sqrt{\frac{2\log(2|\contextspace|T/\delta)}{\numcontext{s-1}(z)}}(\Parm{A_s}[Z=z]-\indicator\{Z_s=z\})\\
    &\leq
    \sqrt{8\log(2|\contextspace|T/\delta)|\contextspace|t}
    +\sum_{s=1}^t\sum_{z\in\contextspace}
    \sqrt{\frac{2\log(2|\contextspace|T/\delta)}{\numcontext{s-1}(z)}}(\Parm{A_s}[Z=z]-\indicator\{Z_s=z\}), 
    \]
    where in the last inequality we use the same argument as in the proof of \cref{regret bound: UCB}, and the remaining summation term can be controlled by $\sqrt{4\log(2|\contextspace|T/\delta)\log(T/\delta')t}$ immediately after we further condition on $\eventmartingale(\delta')$. Therefore, we get
    \*[
    \Reg(t)
    \leq
    \sqrt{\log(2|\contextspace|T/\delta)}\lr{\sqrt{8|\contextspace|}+\sqrt{4\log(T/\delta')}}\sqrt{t}, \forall t\in[T],
    \]
    in event $\eventcontext(\delta)\cap\eventmartingale(\delta')$.
\end{proof}
Combining \cref{regret bound: UCB} with \cref{regret bound: C-UCB} and taking $\delta'=\delta$, we thus verfy
\cref{regret bound: cucb and ucb}.

\section{Proofs of Lower Bounds}\label{appendix sec: proofs of lower bounds}
In this section we give the full proof of \cref{lower bound: pareto lower bound} and \cref{lower bound: when we have zero knowledge of marginals}. Note that our proof of \cref{lower bound: pareto lower bound} mainly adopts but also largely generalizes the one of \citet[][Theorem 6.2]{bilodeau2022adaptively}.
\subsection{Proof of \cref{lower bound: pareto lower bound}}\label{appendix sec: proof of lower bound on the frontier}
\begin{proof}[Proof of \cref{lower bound: pareto lower bound}]
   Fix $\actionspace, \contextspace$ and $T$. Let $\contextspace_0$ be an arbitrary proper subset of $\contextspace$ and $\contextspace_1=\contextspace\setminus\contextspace_0$. Fix $\Delta\in (0,1/20)$ to be chosen later. Define the family of marginals for all instances appearing in this proof
   \*[
   q_a[Z\in\contextspace_0]
   =
   \begin{cases}
       1/2+2\Delta & a=1 \\
       1/2 & a\neq 1,
   \end{cases}
   \]
   where probability is evenly spread within $\contextspace_0$ and $\contextspace_1$ respectively. Then define a conditionally benign environment $\env\in\envspace$ by
   \*[
   \hP_{\env_a}[Y=1]
   =
   \sum_{z\in\contextspace} p[Y=1|Z=z]q_a[Z=z], 
   \quad \forall a\in\actionspace,
   \]
   where $p[Y|Z]$ is a Bernoulli conditional distribution such that
   \*[
   p[Y=1|Z=z]=
   \begin{cases}
       3/4 & z\in\contextspace_0 \\
       1/4 & z\in\contextspace_1.
   \end{cases}
   \]
   Now we define some non-benign instances. For every $a_0\neq 1$, define $\env_{a_0}$ by
   \*[
   \hP_{\env_a^{a_0}}[Y=1]
   =
   \sum_{z\in\contextspace} p_a^{a_0}[Y=1|Z=z]q_a[Z=z],
   \quad \forall a\in\actionspace,
   \]
   where $p_a^{a_0}[Y|Z]$ is a Bernoulli conditional distribution such that
   \*[
   p_a^{a_0}[Y=1|Z=z]
   =
   \begin{cases}
       3/4 & a=1, z\in\contextspace_0 \\
       1/4 & a=1, z\in\contextspace_1 \\
       3/4+4\Delta & a=a_0, z\in\contextspace_0 \\
       1/4 & a=a_0, z\in\contextspace_1 \\
       3/4 & a\notin\{1, a_0\}, z\in\contextspace_0 \\
       1/4 & a\notin\{1, a_0\}, z\in\contextspace_1.
   \end{cases}
   \]
   For any MAB algorithm $\alg$, let $\pi^q=\alg(\actionspace,
   \contextspace,q,T)$ be the actual policy implemented by $\alg$ when it's interacting with $\env$ and $\env^{a_0}$. Then by the divergence decomposition formula and Bretagnolle-Huber inequality,
   \*[
   \hE_{\env,\pi^q}[\Reg(T)] + \hE_{\env^{a_0},\pi^q}[\Reg(T)]
   &\geq
   \frac{T\Delta}{2}\hP_{\env,\pi^q}[\numaction{T}(1)\leq T/2]
   +
   \frac{T\Delta}{2}\hP_{\env^{a_0},\pi^q}[\numaction{T}(1)> T/2] \\
   &\geq
   \frac{T\Delta}{4}\exp(-\KL{\hP_{\env,\pi^q}}{\hP_{\env^{a_0}, \pi^q}}) \\
   &=
   \frac{T\Delta}{4}\exp\lr{-\frac{1}{2}\hE_{\env,\pi^q}[\numaction{T}(a_0)]\KL{\bernoulli{3/4}}{\bernoulli{3/4+4\Delta}}}\\
   &\geq
   \frac{T\Delta}{4}\exp\lr{-\hE_{\env,\pi^q}[\numaction{T}(a_0)]\cdot 32\Delta^2},
   \]
   where in the last step we use $\KL{\bernoulli{3/4}}{\bernoulli{3/4+4\Delta}}\leq 64\Delta^2$ for $\Delta<1/40$. Combined with the worst-case regret upper bound $\hE_{\env,\pi^q}[\Reg(T)] + \hE_{\env^{a_0},\pi^q}[\Reg(T)]\leq 2 R(T;\actionspace,\contextspace)$, it implies that
   \*[
   \hE_{\env,\pi^q}[\numaction{T}(a_0)]
   \geq
   \frac{1}{32\Delta^2}\log\lr{\frac{T\Delta}{8 R(T;\actionspace,\contextspace)}}, \forall a_0\neq 1.
   \]
   Realizing $\hE_{\env,\pi^q}[\Reg(T)]=\sum_{a_0\neq 1} \Delta \hE_{\env,\pi^q}[\numaction{T}(a_0)]$, we have 
   \*[
   \hE_{\env,\pi^q}[\Reg(T)]
   \geq
   \frac{|\actionspace|-1}{32\Delta}\log\lr{\frac{T\Delta}{8 R(T;\actionspace,\contextspace)}}.
   \]
   So there exists absolute constants $c=\log 2/1024,c'=1/641$ such that whenever $R(T;\actionspace,\contextspace)\leq c'T$, the choice of 
   $\Delta=\frac{16 R(T;\actionspace,\contextspace)}{T}$ satisfies $\Delta<1/40$ and 
   \*[
   \hE_{\env,\pi^q}[\Reg(T)]
   \geq
   c\cdot\frac{|\actionspace|T}{R(T;\actionspace,\contextspace)},
   \]
   which completes the proof.
\end{proof}
\subsection{Proof of \cref{lower bound: when we have zero knowledge of marginals}}\label{appendix sec: proof of lower bound when we have zero knowledge of marginals}
\begin{proof}[Proof of \cref{lower bound: when we have zero knowledge of marginals}]
    Fix $\actionspace, \contextspace$ and $T$. Let $\contextspace_0$ be an arbitrary proper subset of $\contextspace$ and $\contextspace_1=\contextspace\setminus\contextspace_0$. Fix $\Delta\in(0,\frac{1}{40})$ to be chosen later. For all conditionally benign instances $\env$ in this proof, we consider $\hP_{\env_a}[Y|Z]$ to be the Bernoulli distribution given by
    \*[
    \hP_{\env_a}[Y=1|Z=z]=p[Y=1|Z=z]=
    \begin{cases}
        3/4
        & z\in\contextspace_0\\
        1/4
        &
        z\in\contextspace_1,
    \end{cases}
    \]
    which implies that contexts from $\contextspace_0$ are more rewarding than those from $\contextspace_1$.

    Now define conditionally benign environments $\env,\env^{a_0}\in\envspace, \forall a_0\neq 1$, through their marginals
    \*[
    \hP_{\env_a}[Y=1]&=
    \sum_{z\in\contextspace}
    p[Y=1|Z=z]q_a[Z=z],\\
    \hP_{\env^{a_0}_a}[Y=1]&=
    \sum_{z\in\contextspace}
    p[Y=1|Z=z]q^{a_0}_a[Z=z], 
    \forall a\in\actionspace
    \]
    where
    \*[
    q_a[Z\in\contextspace_0]=
    \begin{cases}
        1/2+2\Delta
        & a=1\\
        1/2
        &a\neq 1
    \end{cases}
    \quad\text{and}\quad
    q^{a_0}_a[Z\in\contextspace_0]=
    \begin{cases}
        1/2+2\Delta
        &a=1\\
        1/2+4\Delta
        &a=a_0\\
        1/2
        &a\neq 1, a_0,
    \end{cases}
    \]
    where probability is evenly spaced within $\contextspace_0$ and $\contextspace_1$. So clearly action 1 is the only optimal action in $\env$ and action $a_0$ is the only optimal action in $\env^{a_0}$, with sub-optimality gap $\mingap(\env)=\mingap(\env^{a_0})=\Delta$.

    Fix algorithm $\alg\in\noncausalspace$ with $\tilde\policy=\alg(\actionspace,\contextspace, T, \cdot)$ be the actual policy implemented by $\alg$. By the divergence decomposition formula (\citealt{bilodeau2022adaptively}), we have that for every $a_0\neq 1$,
    \*[
    \KL{\hP_{\env,\tilde\policy}}{\hP_{\env^{a_0}, \tilde\policy}}
    &=
    \sum_{a\in\actionspace}\hE_{\env,\tilde\policy}[\numaction{T}(a)]\KL{\hP_{\env_a}}{\hP_{\env^{a_0}_a}}\\
    &=
    \sum_{a\in\actionspace}\hE_{\env,\tilde\policy}[\numaction{T}(a)]\KL{q_a}{q^{a_0}_a}\\
    &=
    \hE_{\env,\tilde\policy}[\numaction{T}(a_0)]\KL{q_{a_0}}{q^{a_0}_{a_0}}\\
    &=
    \hE_{\env,\tilde\policy}[\numaction{T}(a_0)]\KL{\bernoulli{1/2}}{\bernoulli{1/2+4\Delta}}.
    \]
    By Bretagnolle–Huber inequality,
    \*[
    \hE_{\env,\tilde\policy}[\Reg(T)]+\hE_{\env^{a_0},\tilde\policy}[\Reg(T)]
    &\geq
    \frac{T\Delta}{2}\lr{\hP_{\env,\tilde\policy}[\numaction{T}(1)\leq T/2]+\hP_{\env^{a_0},\tilde\policy}[\numaction{T}(1)>T/2]} \\
    &\geq
    \frac{T\Delta}{4}\exp(-\KL{\hP_{\env,\tilde\policy}}{\hP_{\env^{a_0}, \tilde\policy}})\\
    &=
    \frac{T\Delta}{4}\exp(-\hE_{\env,\tilde\policy}[\numaction{T}(a_0)]\KL{\bernoulli{1/2}}{\bernoulli{1/2+4\Delta}}).
    \]
    Now we pick $a_0\in\cArgmin_{a\neq 1}\hE_{\env,\tilde\policy}[\numaction{T}(a)]$ which implies that $\hE_{\env,\tilde\policy}[\numaction{T}(a_0)]\leq\frac{T}{|\actionspace|-1}$. Also $\KL{\bernoulli{1/2}}{\bernoulli{1/2+4\Delta}}\leq 4(4\Delta)^2=64\Delta^2$ for $\Delta<1/40$. So
    \*[
    \hE_{\env,\tilde\policy}[\Reg(T)]+\hE_{\env^{a_0},\tilde\policy}[\Reg(T)]
    \geq
    \frac{T\Delta}{4}\exp\lr{-\frac{64T\Delta^2}{|\actionspace|-1}}.
    \]
    Taking $\Delta=\frac{1}{40}\sqrt{\frac{|\actionspace|-1}{T}}$, we know that $\max\lrset{\hE_{\env,\tilde\policy}[\Reg(T)], \hE_{\env^{a_0},\tilde\policy}[\Reg(T)]}\geq\frac{1}{2}(\hE_{\env,\tilde\policy}[\Reg(T)]+\hE_{\env^{a_0},\tilde\policy}[\Reg(T)])\geq c\sqrt{|\actionspace|T}$ for some absolute constant $c>0$, which yields the claim.
\end{proof}
In the above proof, it is easy to see that $\env(Z)$ and $\env^{a_0}(Z)$ are $\varepsilon-$close, where $\varepsilon=c\cdot\sqrt{|\actionspace|/T}$ for some absolute constant $c$. So for any algorithm input by $\tilde\env(Z)=\env(Z)$ when interacting with $\bar\env\in\lrset{\env,\env^{a_0}, a_0\neq 1}$, it is satisfied that $\tilde\env(Z)$ and $\bar\env(Z)$ are always $\varepsilon-$close, but the algorithm incurs $\Omega(\sqrt{|\actionspace|T})$ regret in some instance from $\lrset{\env,\env^{a_0}, a_0\neq 1}$.

\section{Instances where \texorpdfstring{$\elim$}{PE} incurs linear regret}\label{appendix sec: proofs regarding adaptive ins-dep results}
In this section we give an example for $\elim$ to illustrate that to merely force linear regret on a causal bandit algorithm, we need to construct non-benign instances carefully and re-code the algorithm to ensure its erratic behavior in those instances. In particular, we construct a non-benign environment $\env$ for every $\Delta\in(0,1)$ such that $\mingap(\env)=\Delta$ while the re-coded $\elim$ never plays the optimal arm.
\begin{proposition}
\label{fact: phased elimination incurs linear regret in some non-benign settings}
    Suppose we modify \cref{alg: phased-elimination} such that, in each phase, we always  choose an exact-optimal design whenever feasible.
    For any $\actionspace,\contextspace$ and $T$ with $|\actionspace|>|\contextspace|\geq 3$ and $\Delta\in(0,1)$, there exists a non-benign environment $\env$ such that $\mingap(\env)=\Delta$, while
    \cref{alg: phased-elimination} will never play the optimal arm, hence incurring linear regret,
    \*[
    \Reg(t)
    \geq
    \mingap(\env)\cdot t
    =\Delta\cdot t,
    \quad \forall t\in[T].
    \]
\end{proposition}

\begin{proof}[Proof of \cref{fact: phased elimination incurs linear regret in some non-benign settings}]
    For any $\actionspace$ and $\contextspace$ with $|\actionspace|>|\contextspace|\geq 3$, suppose we index the contexts in arbitrary way such that $\contextspace=\lrset{z_1,...,z_{|\contextspace|}}$, and we pick $(|\contextspace|+1)$ number of arms from $\actionspace$ and denote them by $\optarm, a_1,...,a_{|\contextspace|}$. Construct marginals $\env_a$ as follows:
     \*[
     \env_{a_i}(Z) &= \delta_{\lrset{z_i}} =: e_i, \quad i\in[|\contextspace|], \\
     \env_{\optarm}(Z) &= \frac{1}{2}(\delta_{\lrset{z_1}}+\delta_{\lrset{z_2}}) = \frac{1}{2}(e_1 + e_2),
     \]
     where we write marginal distributions over $\contextspace$ as vectors in $\mathbb R^{|\contextspace|}$ according to context indices. Then define conditional distributions $\hP_{\env_a}(Y|Z)$:
     \*[
     \hP_{\env_{a_i}}[Y|Z=z_i] =
     \begin{cases}
     \delta_{\lrset{0}} & i\in[|\contextspace|-1] \\
     \delta_{\lrset{1-\Delta}} & i = |\contextspace|
     \end{cases}
     \]
     and
     \*[
     \hP_{\env_{\optarm}}[Y|Z=z_1]&=\hP_{\env_{\optarm}}[Y|Z=z_2] = \delta_{\lrset{1}}.
     \]
     In other words, playing arm $a_i$ yields context $z_i$ and deterministic reward, while we could observe $z_1$ or $z_2$ with equal probability and always get the optimal reward by playing arm $\optarm$. So the only optimal arm for $\env$ is $\optarm$ with $\mingap(\env)=\Delta$. We can treat all other $a\in\actionspace$ as dummy actions by identifying each of them with one of $\optarm, a_i, i\in[|\contextspace|]$ arbitrarily. 

     Next we will verify the following facts. (1) When no action is eliminated and $\activeset{\ell}=\actionspace$, any exact G-optimal design $\phasedesign{\ell}\in\cP(\activeset{\ell})$ does not have positive mass over $\optarm$. (2) Whenever any action is eliminated in the end of phase $\ell$, it must be that all actions except for $a_{|\contextspace|}$ are eliminated as well. Then PE would just play $a_{|\contextspace|}$ till the end. Combining these two facts we can conclude that PE never picks $\optarm$ during the interaction with $\env$.
     \paragraph{No G-optimal design is supported on $\optarm$.} Recall that any G-optimal design $\phasedesign{\ell}$ maximizes $f(\pi)=\log\det V(\pi)$, where $V(\pi)=\sum_{a\in\activeset{\ell}}\pi(a)\nu_a\nu_a^\top$ over $\pi\in\cP(\activeset{\ell})$ \citep[][Theorem~21.1]{lattimore2020bandit}. When $\activeset{\ell}=\actionspace$, $\det V(\pi)$ can be computed as
     \*[
     \det V(\pi)
     =
     \lr{\pi(a_1)\pi(a_2)+ \frac{\pi(\optarm)}{4}(\pi(a_1)+\pi(a_2))}\pi(a_3)\cdots\pi(a_{|\contextspace|}).
     \]
     Then we can find that any maximizing $\pi$ should have $\pi(\optarm)=0$ after realizing that $\pi(a_1)=\pi(a_2)$ for such $\pi$. Moreover, there is only one G-optimal design in this case, which is $\phasedesign{\ell}=Unif(a_1,...,a_{|\contextspace|})$.

     \paragraph{All actions other than $a_{|\contextspace|}$ would be eliminated at the same time.} If the first elimination happens in the end of phase $\ell$, then we must have $\empmeancontext{\ell}=(0,...,0, 1-\Delta)^\top$ due to that $\phasedesign{\ell}=Unif(a_1,...,a_{|\contextspace|})$ and rewards are deterministic. So $\max_{b\in\activeset{\ell}} \ip{\phaseestimate{\ell}, \nu_b-\nu_a}$ is $1-\Delta$ for all $a\neq a_{|\contextspace|}$ and $0$ for $a=a_{|\contextspace|}$. Then the elimination must happen within $\optarm, a_1,...,a_{|\contextspace|-1}$, and thus every one of it should be eliminated simultaneously.
\end{proof}

\end{document}